\documentclass[11pt]{article}

\usepackage{amsmath,amsthm,amsfonts,amscd,amssymb,bm,mathrsfs}
\usepackage{abstract}
\usepackage{enumerate}
\usepackage{graphicx,xcolor}
\usepackage{epstopdf}
\usepackage{multirow}
\usepackage{cases}
\usepackage[normalem]{ulem}
\usepackage{indentfirst}
\usepackage{geometry}
\usepackage{subfigure}
\useunder{\uline}{\ul}{}

\newcommand{\R}{{\mathbb{R}}}
\newcommand{\N}{{\mathbb{N}}}

\newtheorem{Lemma}{Lemma}
\newtheorem{Definition}{Definition}
\newtheorem{Proposition}{Proposition}
\newtheorem{Remark}{Remark}
\newtheorem{Theorem}{Theorem}
\newtheorem{Corollary}{Corollary}
\newtheorem{Example}{Example}

\title{Sparse estimation via $\ell_q$ optimization method in high-dimensional linear regression}
\author{}
\date{}

\author{
Xin Li\thanks{School of Mathematical Sciences, Zhejiang University, Hangzhou 310027, P. R. China (11435017@zju.edu.cn).}
\and Yaohua Hu\thanks{Shenzhen Key Laboratory of Advanced Machine Learning and Applications, College of Mathematics and Statistics, Shenzhen University, Shenzhen 518060, P. R. China (mayhhu@szu.edu.cn).}
\and Chong Li\thanks{School of Mathematical Sciences, Zhejiang University, Hangzhou 310027, P. R. China (cli@zju.edu.cn).}
\and Xiaoqi Yang\thanks{Department of Applied Mathematics, The Hong Kong Polytechnic University, Kowloon, Hong Kong (mayangxq@polyu.edu.hk).}
\and Tianzi Jiang\thanks{Brainnetome Center, Institute of Automation, Chinese Academy of Sciences, Beijing 100190, P. R. China (jiangtz@nlpr.ia.ac.cn).}
}

\begin{document}
\maketitle

\begin{abstract}
In this paper, we discuss the statistical properties of the $\ell_q$ optimization methods $(0<q\leq 1)$, including the $\ell_q$ minimization method and the $\ell_q$ regularization method, for estimating a sparse parameter from noisy observations in high-dimensional linear regression with either a deterministic or random design. For this purpose, we introduce a general $q$-restricted eigenvalue condition (REC) and provide its sufficient conditions in terms of several widely-used regularity conditions such as sparse eigenvalue condition, restricted isometry property, and mutual incoherence property. By virtue of the $q$-REC, we exhibit the stable recovery property of the $\ell_q$ optimization methods for either deterministic or random designs by showing that the $\ell_2$ recovery bound $O(\epsilon^2)$ for the $\ell_q$ minimization method and the oracle inequality and $\ell_2$ recovery bound $O(\lambda^{\frac{2}{2-q}}s)$ for the $\ell_q$ regularization method hold respectively with high probability. The results in this paper are nonasymptotic and only assume the weak $q$-REC. The preliminary numerical results verify the established statistical property and demonstrate the advantages of the $\ell_q$ regularization method over some existing sparse optimization methods.
\end{abstract}


\section{Introduction}

\noindent In various areas of applied sciences and engineering, a fundamental problem is to estimate an unknown parameter $\beta^*\in \R^n$ of a linear regression model
\begin{equation}\label{model}
y=X\beta^*+e,
\end{equation}
where $X\in \R^{m\times n}$ is a design matrix, $e\in \R^m$ is a vector containing random measurement noise, and thus $y\in \R^m$ is the corresponding vector of the noisy observations. According to the context of practical applications, the design matrix could be either deterministic or random.\\
\indent The curse of dimensionality always occurs in the high-dimensional regime of many application fields. For example, in magnetic resonance imaging \cite{Candes2006Stable}, remote sensing \cite{Aronoff2004Remote}, systems biology \cite{Qin2014Inferring}, one is typically only able to collect far fewer samples than the number of variables due to physical or economical constraints, i.e., $m\ll n$. Under the high-dimensional scenario, estimating the true underlying parameter of model \eqref{model} is a vital challenge in contemporary statistics, whereas the classical ordinary least squares (OLS) does not work well in this scenario because the corresponding linear system is seriously ill-conditioned.

\subsection{$\ell_1$ Optimization Problems}
\noindent Fortunately, in practical applications, a wide class of problems usually have certain special structures, employing which could eliminate the nonidentifiability of model \eqref{model} and enhance the predictability. One of the most popular structures is the sparsity structure, that is, the underlying parameter $\beta^*$ in the high-dimensional space is sparse. 
One common way to measure the degree of sparsity is the $\ell_q$ norm, which for $0<q\le 1$ is defined as
\begin{equation*}
\|\beta\|_q:=\left(\sum_{i=1}^n |\beta_i|^q\right)^{1/q},
\end{equation*}
while $\|\beta\|_0$ is defined as the number of nonzero entries of $\beta$. We first review the literature of sparse estimation for the case when the design matrix $X$ is deterministic. In the presence of a bounded noise (i.e., $\|e\|_2\leq \epsilon$), in order to find the sparest solution, Donoho et al. \cite{Donoho2006Stable} proposed the following (constrained) $\ell_0$ minimization problem:
\begin{equation*}
(\text{CP}_{0,\epsilon})\quad \min\, \|\beta\|_0\quad \text{s.t.}  \quad \|y-X\beta\|_2\leq \epsilon.
\end{equation*}
Unfortunately, it is NP-hard to compute its global solution due to the nonconvex and combinational natures \cite{Natarajan1995Sparse}.

To overcome this obstacle, a common technique is to use the (convex) $\ell_1$ norm to approach the $\ell_0$ norm:
\begin{equation*}
(\text{CP}_{1,\epsilon})\quad \min\, \|\beta\|_1\quad \text{s.t.}  \quad \|y-X\beta\|_2\leq \epsilon,
\end{equation*}
which can be efficiently solved by several standard methods; see \cite{Chen2001Atomic, Figueiredo2007Gradient} and references therein. The stable statistical properties of $(\text{CP}_{1,\epsilon})$ have been explored under the regularity conditions. One of the most important stable statistical properties is the $\ell_2$ recovery bound property, which is to estimate the upper bound of the error between the optimal solution of the optimization problem and the true underlying parameter in terms of the noise level $\epsilon$. More specifically, let $s\ll n$ and $\beta^*$ be an $s$-sparse parameter (i.e., $\|\beta^*\|_0\leq s$) satisfying the linear regression model \eqref{model}. The $\ell_2$ recovery bound for $(\text{CP}_{1,\epsilon})$ was provided in \cite{Donoho2006Stable} and \cite{Candes2006Stable} under the mutual incoherence property (MIP) or the restricted isometry property (RIP)\footnote{It was claimed in \cite{Cai2009Recovery} that the RIP \cite{Candes2005Decoding} is implied by the MIP \cite{Donoho2001Uncertainty}, while the restricted isometry constant (RIC) is more difficult to be calculated than the mutual incoherence constant (MIC).}, respectively:
\begin{equation*}
\|\bar{\beta}_{1,\epsilon}-\beta^*\|_2=O(\epsilon),
\end{equation*}
where $\bar{\beta}_{1,\epsilon}$ stands for the optimal solution of $(\text{CP}_{1,\epsilon})$. 

In some applications, the amplitude of noise is difficult to estimate. As such the study of the constrained sparse optimization models is underdeveloped. In such situations, the regularization technique has been widely used in statistics and machine learning, which helps to avoid the noise estimation by introducing a regularization parameter. Specifically, one solves the (unconstrained) $\ell_1$ regularization problem:
\begin{equation*}
(\text{RP}_{1,\lambda})\quad \min\, \frac{1}{2m}\|y-X\beta\|_2^2+\lambda\|\beta\|_1,
\end{equation*}
where $\lambda>0$ is the regularization parameter, providing a tradeoff between data fidelity and sparsity. The $\ell_1$ regularization model, also named the Lasso estimator \cite{Tibshirani1996Regression}, has attracted a great deal of attention in parameter estimation in the high-dimensional scenario, because its convexity structure is beneficial in designing exclusive and efficient algorithms and gaining wide applications;
see \cite{Beck2009Fast, Daubechies2010Iteratively} and references therein. For the noise-free case,
the $\ell_2$ recovery bound for $(\text{RP}_{1,\lambda})$ was provided in \cite{Geer2009On} under the RIP or the restricted eigenvalue condition (REC)\footnote{It was reported in \cite{Bickel2009Simultaneous} that the REC is implied by the RIP, and in \cite{Raskutti2010Restricted} that a broad class of correlated Gaussian design matrices satisfy the REC but violate the RIP with high probability.}:
\begin{equation*}
\|\hat{\beta}_{1,\lambda}-\beta^*\|_2^2=O(\lambda^2s),
\end{equation*}
where $\hat{\beta}_{1,\lambda}$ denotes the optimal solution of $(\text{RP}_{1,\lambda})$.
Furthermore, assuming that the noise in model \eqref{model} is normally distributed $e\sim \mathscr{N}(0,\sigma^2\mathbb{I}_m)$, it was established in \cite{Bickel2009Simultaneous, Bunea2007Sparsity, Zhang2009Some} that the following $\ell_2$ recovery bound holds with high probability
\begin{equation*}
\|\hat{\beta}_{1,\lambda}-\beta^*\|_2^2=O\left(\sigma^2s\frac{\log n}{m}\right),
\end{equation*}
when the regularization parameter is chosen as $\lambda=\sigma\sqrt{\frac{\log n}{m}}$ and under the RIP, REC or other regularity conditions, respectively. However, the $\ell_1$ minimization and regularization problems suffer several dissatisfactions in both theoretical and practical applications. In particular, it was reported by extensive theoretical and empirical studies that the $\ell_1$ minimization and regularization problems suffer from significant estimation bias when parameters have large absolute values; the induced solutions are much less sparse than the true parameter, they cannot recover a sparse signal with the least samples when applied to compressed sensing, and that they often result in sub-optimal sparsity in practice; see, e.g., \cite{Chartrand2007Exact, Fan2001Variable, Zhang2010Nearly, Xu2012Regu, Zhang2010Analysis}. Therefore, there is a great demand for developing the alternative sparse estimation technique that enjoys nice statistical theory and successful applications.

To address the bias and the sub-optimal issues induced by the $\ell_1$ norm, several nonconvex regularizers have been proposed such as the smoothly clipped absolute deviation (SCAD) \cite{Fan2001Variable}, minimax concave penalty (MCP) \cite{Zhang2010Nearly}, $\ell_0$ norm \cite{Zhang2012General}, $\ell_q$ norm ($0<q<1$) \cite{Foucart2009Sparsest}, and capped $\ell_1$ norm \cite{Loh2015Regularized}; specifically, the SCAD and MCP fall into the category of folded concave penalized (FCP) methods. It was studied in \cite{Zhang2012General} that the global solution of the FCP sparse linear regression enjoys the oracle property under the sparse eigenvalue condition; see Remark 4(iii) for details.

It is worth noting that the $\ell_q$ norm regularizer ($0<q<1$) has been recognized as an important technique for sparse optimization and gained successful applications in various applied science fields; see, e.g., \cite{Chartrand2007Exact, Qin2014Inferring, Xu2012Regu}. In the present paper, we focus on the statistical property of the $\ell_q$ optimization method, which is beyond the category of the FCP. Throughout the whole paper, we always assume that $0<q\le 1$ unless otherwise specified.

\subsection{$\ell_q$ Optimization Problems}

\noindent Due to the fact that $\lim_{q\to 0^+}\|\beta\|_q^q=\|\beta\|_0$, the $\ell_q$ norm has also been adopted as another alternative sparsity promoting penalty function of the $\ell_0$ and $\ell_1$ norms. The following $\ell_q$ optimization problems have attracted a great amount of attention and gained successful applications in a wide range of fields (see \cite{Chartrand2007Exact, Qin2014Inferring, Xu2012Regu} and references therein):
\begin{equation*}
(\text{CP}_{q,\epsilon})\quad \min\, \|\beta\|_q\quad \text{s.t.}  \quad \|y-X\beta\|_2\leq \epsilon,
\end{equation*}
and
\begin{equation*}
(\text{RP}_{q,\lambda})\quad \min\, \frac{1}{2m}\|y-X\beta\|_2^2+\lambda\|\beta\|_q^q.
\end{equation*}
In particular, the numerical results in \cite{Chartrand2007Exact} and \cite{Xu2012Regu} showed that the $\ell_q$ minimization and the $\ell_{\frac12}$ regularization admit a significantly stronger sparsity promoting capability than the $\ell_1$ minimization and the $\ell_1$ regularization, respectively; that is, they allow to obtain a more sparse solution from a smaller amount of samplings. \cite{Qin2014Inferring} revealed that the $\ell_{\frac12}$ regularization achieved a more reliable biological solution than the $\ell_1$ regularization in the field of systems biology.

The advantage of the lower-order optimization problem has also been shown in theory that it requires a weaker regularity condition to guarantee the stable statistical property than the classical $\ell_1$ optimization problem. In particular, let $\bar{\beta}_{q,\epsilon}$ and $\hat{\beta}_{q,\lambda}$ denote the optimal solution of $(\text{CP}_{q,\epsilon})$ and $(\text{RP}_{q,\lambda})$, respectively. The $\ell_2$ recovery bound for $(\text{CP}_{q,\epsilon})$ was established in \cite{Dong2016Unified} and \cite{Song2014Sparse} under MIP and RIP respectively:
\begin{equation}\label{eq-CP-rb}
\|\bar{\beta}_{q,\epsilon}-\beta^*\|_2=O(\epsilon),
\end{equation}
where the MIP or RIP is weaker than the one used in the study of $(\text{CP}_{1,\epsilon})$. \cite{Hu2017Group} established an $\ell_2$ recovery bound for $(\text{RP}_{q,\lambda})$ in the noise-free case:
\begin{equation}\label{bound-Hu}
\|\hat{\beta}_{q,\lambda}-\beta^*\|_2^2=O(\lambda^{\frac{2}{2-q}}s)
\end{equation}
under the introduced $q$-REC, which is strictly weaker than the classical REC. However, the theoretical study of the $\ell_q$ optimization problem is still limited; particularly, there is still no paper devoted to establishing the statistical property of the $\ell_q$ minimization problem when the noise is randomly distributed, and that of the $\ell_q$ regularization problem in the noise-aware case.

\subsection{Contributions of This Paper}
\noindent The main contribution of the present paper is the establishment of the statistical properties for the $\ell_q$ optimization problems, including $(\text{CP}_{q,\epsilon})$ and $(\text{RP}_{q,\lambda})$, in the noise-aware case; specifically, in the case when the linear regression model \eqref{model} involves a Gaussian noise $e\sim \mathscr{N}(0,\sigma^2\mathbb{I}_m)$.
For this purpose, we extend the $q$-REC \cite{Hu2017Group} to a more general one, which is one of the weakest regularity conditions for estimating the $\ell_2$ recovery bounds of sparse estimation models, and provide some sufficient conditions for guaranteeing the general $q$-REC in terms of REC, RIP, and MIP (with a less restrictive constant); see Propositions \ref{prop-REC} and \ref{prop-RECnew}. Under the general $q$-REC, we show that the $\ell_2$ recovery bound \eqref{eq-CP-rb} holds for $(\text{CP}_{q,\epsilon})$ with high probability, and that
\[\|\hat{\beta}_{q,\lambda}-\beta^*\|_2^2=O\left(\left(\sigma^2 \frac{\log n}{m}\right)^{\frac{1}{2-q}}s\right),\]
as well as the estimation of prediction loss and the oracle property, hold for $(\text{RP}_{q,\lambda})$ with high probability; see Theorems \ref{thm-CP} and \ref{thm-RP}, respectively. These results provide a unified framework of the statistical properties of the $\ell_q$ optimization problems, and improve the ones of the $\ell_q$ minimization problem \cite{Dong2016Unified, Song2014Sparse} and the $\ell_1$ regularization problem \cite{Bickel2009Simultaneous, Bunea2007Sparsity, Zhang2009Some} under the $q$-{\rm REC}; see Remark \ref{rmk-thm2}.
They are not only of independent interest in establishing statistical properties for the lower-order optimization problems with randomly noisy data, but also provide a useful tool for the study of the case when the design matrix $X$ is random.

Another contribution of the present paper is to explore the $\ell_2$ recovery bounds for the $\ell_q$ optimization problems with a random design matrix $X$ and random noise $e$, which is more realistic in the real-world applications; e.g., compressed sensing \cite{Candes2006Robust}, signal processing \cite{Candes2006Stable}, statistical learning \cite{Agarwal2012Fast}.
As reported in \cite{Raskutti2010Restricted}, the key issue for studying the statistical properties of a sparse estimation model with a random design matrix is to provide suitable conditions on the population covariance matrix $\Sigma$ of $X$, which can guarantee the regularity conditions with high probability; see, e.g., \cite{Candes2006Stable, Raskutti2010Restricted}.
Motivated by the real-world applications, we consider the standard case when $X$ is a Gaussian random design with i.i.d. $\mathscr{N}(0,\Sigma)$ rows and the linear regression model \eqref{model} involves a Gaussian noise, explore a sufficient condition for ensuring the $q$-REC of $X$ with high probability in terms of the $q$-REC of $\Sigma$, and apply the preceding results to establish the $\ell_2$ recovery bounds \eqref{eq-CP-rb} for $(\text{CP}_{q,\epsilon})$, and \eqref{bound-Hu}, as well as the predication loss and the oracle inequality, for $(\text{RP}_{q,\lambda})$, respectively; see Theorems \ref{thm-CP-randomX} and \ref{thm-RP-randomX}. These results provide a unified framework of the statistical properties of the $\ell_q$ optimization problems with a Gaussian random design under the $q$-{\rm REC}, which cover the ones of the $\ell_1$ optimization problems (see \cite[Theorem 3.1]{Zhou2009Restricted}) as special cases; see Corollaries \ref{corol-CP-randomX} and \ref{corol-RP-randomX}.
To the best of our knowledge, most results presented in this paper are new, either for the deterministic or random design matrix.

We also carry out the numerical experiments on the standard simulated data. The preliminary numerical results verify the established statistical properties and show that the $\ell_q$ optimization methods possess better recovery performance than the $\ell_1$ optimization method, SCAD and MCP, which coincides with existing numerical studies \cite{Hu2017Group, Xu2012Regu} on the $\ell_q$ regularization problem. More specifically, the $\ell_q$ regularization method outperforms the $\ell_1$, SCAD and MCP regularization methods in the sense that its estimated error decreases faster when the sample size increases and achieves a more accurate solution.

The remainder of this paper is organized as follows. In section 2, we introduce the lower-order REC and discuss its sufficient conditions. In section 3, we establish the $\ell_2$ recovery bounds for $(\text{CP}_{q,\epsilon})$ and $(\text{RP}_{q,\lambda})$ with a deterministic design matrix. The extension to the linear regression model with a Gaussian random design and preliminary numerical results are presented in sections 4 and 5, respectively. \\
\indent We end this section by presenting the notations adopted in this paper. We use Greek lowercase letters $\alpha,\beta,\delta$ to denote the vectors, capital letters $J$, $T$ to denote the index sets, and script captical letters $\mathscr{A}$, $\mathscr{B}$, $\mathscr{C}$ to denote the random events. For $\beta\in \R^n$ and $J\subseteq \{1,2,\dots,n\}$, we use $\beta_J$ to denote the vector in $\R^n$ that $(\beta_J)_i=\beta_i$ for $i\in J$ and zero elsewhere, $|J|$ to denote the cardinality of $J$, $J^c:=\{1,2,\dots,n\}\setminus J$ to denote the complement of $J$, and ${\rm supp}(\beta)$ to denote the support of $\beta$, i.e., the index set of nonzero entries of $\beta$. Particularly,
$\mathbb{I}_m$ stands for the identity matrix in $\R^m$, and
$\mathbb{P}(\mathscr{A})$ and $\mathbb{P}(\mathscr{A}|\mathscr{B})$ denote the probability that event $\mathscr{A}$ happens and the conditional probability that event $\mathscr{A}$ happens given that event $\mathscr{B}$ happens, respectively.

\section{Restricted Eigenvalue Conditions}
\noindent This section aims to discuss some regularity conditions imposed on the design matrix $X$ that are needed to guarantee the stable statistical properties of $(\text{\text{CP}}_{q,\epsilon})$ and $(\text{\text{RP}}_{q,\lambda})$.\\
\indent In statistics, the ordinary least squares (OLS) is a classical technique for estimating the unknown parameters in a linear regression model and has favourable properties if some regularity conditions are satisfied; see, e.g., \cite{Rao1973Linear}. For example, the OLS always requires the positive definiteness of the Gram matrix
$\Gamma(X):=X^\top X$,
that is,
\begin{equation}\label{I-9}
\min_{\beta\in \R^n:\beta\neq 0}\frac{(\beta^\top \Gamma(X)\beta)^{1/2}}{\|\beta\|_2}= \min_{\beta\in \R^n:\beta\neq 0}\frac{\|X\beta\|_2}{\|\beta\|_2}>0.
\end{equation}
However, in the high-dimensional setting, the OLS does not work well; in fact, the matrix $\Gamma(X)$ is seriously degenerate, i.e.,
\begin{equation*}
\min_{\beta\in \R^n:\beta\neq 0}\frac{\|X\beta\|_2}{\|\beta\|_2}=0.
\end{equation*}
To deal with the challenges caused by the high-dimensional data, the Lasso (least absolute shrinkage and selection operator) estimator was introduced by \cite{Tibshirani1996Regression}. Since then the Lasso estimator has gained a great success in the sparse representation and machine learning of high-dimensional data; see, e.g., \cite{Bickel2009Simultaneous, Geer2008High, Zhang2009Some} and references therein.
It was pointed out that Lasso requires a weak condition, called the restricted eigenvalue condition (REC) \cite{Bickel2009Simultaneous}, to ensure the nice statistical properties; see, e.g., \cite{Geer2009On, Loh2012High, Negahban2012Unified}. In the definition of REC, the minimum in \eqref{I-9} is replaced by a minimum over a restricted set of vectors measured by an $\ell_1$ norm inequality, and the norm $\|\beta\|_2$ in the denominator is replaced by the $\ell_2$ norm of only a part of $\beta$. The notion of REC was extended to the group-wised lower-order REC in \cite{Hu2017Group}, which was used there to explore the oracle property and $\ell_2$ recovery bound of the $\ell_{p,q}$ regularization problem in a noise-free case.\\
\indent Inspired by the ideas in \cite{Bickel2009Simultaneous, Hu2017Group}, we here introduce a lower-order REC for the $\ell_q$ optimization problems, similar to but more general than the one in \cite{Hu2017Group}, where the minimum is taken over a restricted set of vectors measured by an $\ell_q$ norm inequality. To proceed, we shall introduce some useful notations. For the remainder of this paper, let $a>0$ and $(s,t)$ be a pair of integers such that
\begin{equation}\label{s-t}
1\leq s\leq t\leq n \quad \text{and} \quad s+t\leq n.
\end{equation}
For $\delta\in \R^n$ and $J\subseteq \{1,2,\dots,n\}$, we define by $J(\delta;t)$ the index set corresponding to the first $t$ largest coordinates in absolute value of $\delta$ in $J^c$.
For $X\in \R^{m\times n}$, its $q$-restricted eigenvalue modulus relative to $(s,t,a)$ is defined by
\begin{equation}\label{eq-rec}
\phi_q(s,t,a,X):= \min \left\{\frac{\|X\delta\|_2}{\|\delta_{J\cup J(\delta;t)}\|_2}: |J|\leq s, \|\delta_{J^c}\|_q^q\leq a\|\delta_{J}\|_q^q\right\}.
\end{equation}
The lower-order REC is defined as follows.
\begin{Definition}\label{def-rec}
Let $0\le q\le 1$ and $X\in \R^{m\times n}$. $X$ is said to satisfy the $q$-restricted eigenvalue condition relative to $(s,t,a)$ ($q$-\emph{REC}$(s,t,a)$ in short) if \begin{equation*}
\phi_q(s,t,a,X)>0.
\end{equation*}
\end{Definition}
\begin{Remark}\label{rmk-sqrtm}
{\rm (i)} Clearly, the $q$-{\rm REC}$(s,t,a)$ provides a unified framework of the {\rm REC}-type conditions, e.g., it includes the classical {\rm REC} in \cite{Bickel2009Simultaneous} (when $q=1$) and the $q$-{\rm REC}$(s,t)$ in \cite{Hu2017Group} (when $a=1$) as special cases.

{\rm (ii)} The restricted eigenvalue modulus (with $q=1$) defined in \eqref{eq-rec} is slightly different from the one of the classical \emph{REC} in \cite{Bickel2009Simultaneous}, in which the factor $\sqrt{m}$ appears in the denominator there. The reason is that we consider not only the linear regression with a deterministic design as in \cite{Bickel2009Simultaneous}, but also a random design case; for the later case, the $q$-\emph{REC} is assumed to be satisfied for the population covariance matrix of $X$, in which the sample size $m$ does not appear. Hence, to make it consistent for both two cases, we introduce a new definition of the restricted eigenvalue modulus in \eqref{eq-rec} by removing the factor $\sqrt{m}$ from the denominator. Hereby, this is the difference between the restricted eigenvalue modulus \eqref{eq-rec} and that in \cite{Bickel2009Simultaneous}. For example, if the matrix $X$ has i.i.d. Gaussian entries, the restricted eigenvalue modulus in \cite{Bickel2009Simultaneous} scales as a constant, equally, $\phi_q(s,t,a,X)$ given by \eqref{eq-rec} scales as $\sqrt{m}$, independent of $s$, $m$, and $n$, whenever $\frac{s}m \log n$ is bounded. Consequently, the terms in the denominator of conclusions of Theorem \ref{thm-RP} and Corollary \ref{corol-RP} scale as a constant in this situation.
\end{Remark}

It is natural to study the relationships between the $q$-RECs and other types of regularity conditions. To this end, we first recall some basic properties of the $\ell_q$ norm in the following lemmas; particularly, Lemma \ref{lem-1} is taken from \cite[Section 8.12]{Herman2000Equations} and \cite[Lemmas 1 and 2]{Hu2017Group}.
\begin{Lemma}\label{lem-1}
Let $\alpha,\beta\in \R^n$. Then the following relations are true:
\begin{equation}\label{I-1}
\|\beta\|_{q_2}\le \|\beta\|_{q_1} \leq n^{\frac{1}{q_1}-\frac{1}{q_2}}\|\beta\|_{q_2} \quad \mbox{for any } 0<q_1\leq q_2<+\infty,
\end{equation}
\begin{equation}\label{I-2}
\|\alpha\|_q^q-\|\beta\|_q^q\leq \|\alpha+\beta\|_q^q\leq \|\alpha\|_q^q+\|\beta\|_q^q \quad \mbox{for any } 0<q\leq 1.
\end{equation}
\end{Lemma}
\begin{Lemma}\label{lem-2}
Let $p\geq 1$, $n_1,n_2\in \N$, $\alpha\in \R_+^{n_1}$, $\beta\in \R_+^{n_2}$
and $c>0$ be such that
\begin{equation}\label{I-4}
\max\limits_{1\leq i\leq n_1}\alpha_i\leq \min\limits_{1\leq j\leq n_2}\beta_j\quad\text{and}\quad \sum\limits_{i=1}^{n_1}\alpha_i\leq c\sum\limits_{j=1}^{n_2}\beta_j.
\end{equation}
Then
\begin{equation}\label{I-5}
\sum\limits_{i=1}^{n_1}\alpha_i^p\leq c\sum\limits_{j=1}^{n_2}\beta_j^p.
\end{equation}
\end{Lemma}
\begin{proof}
Let $\alpha_{\max}:=\max\limits_{1\leq i\leq n_1}\alpha_i$ and $\beta_{\min}:=\min\limits_{1\leq j\leq n_2}\beta_j$. Then it holds that
\begin{equation}\label{I-6}
\alpha_{\max}\sum_{i=1}^{n_1}\alpha_i^p\leq \alpha_{\max}^p\sum_{i=1}^{n_1}\alpha_i\quad \mbox{and}\quad
\beta_{\min}^p\sum_{j=1}^{n_2}\beta_j\leq \beta_{\min}\sum_{j=1}^{n_2}\beta_j^p.
\end{equation}
Without loss of generality, we assume that $\alpha_{\max}>0$; otherwise, \eqref{I-5} holds automatically. Thus, by the first inequality of \eqref{I-4} and noting $p\geq 1$, we have that
\begin{equation}\label{I-8}
0<\alpha_{\max}^p\beta_{\min}\leq \alpha_{\max}\beta_{\min}^p.
\end{equation}
Multiplying the inequalities in \eqref{I-6} by $\beta_{\min}\sum\limits_{j=1}^{n_2}\beta_j$ and $\alpha_{\max}\sum\limits_{i=1}^{n_1}\alpha_i$ respectively, we obtain that
\begin{equation*}
\begin{aligned}
\alpha_{\max}\beta_{\min}\sum_{i=1}^{n_1}\alpha_i^p\sum_{j=1}^{n_2}\beta_j &\leq \alpha_{\max}^p\beta_{\min}\sum_{i=1}^{n_1}\alpha_i\sum_{j=1}^{n_2}\beta_j\\
&\leq
\alpha_{\max}\beta_{\min}^p\sum_{i=1}^{n_1}\alpha_i\sum_{j=1}^{n_2}\beta_j\\
&\leq
\alpha_{\max}\beta_{\min}\sum_{i=1}^{n_1}\alpha_i\sum_{j=1}^{n_2}\beta_j^p,
\end{aligned}
\end{equation*}
where the second inequality follows from \eqref{I-8}. This, together with the second inequality of \eqref{I-4}, yields \eqref{I-5}. The proof is complete.
\end{proof}
\indent Extending \cite[Proposition 5]{Hu2017Group} to the general $q$-REC, the following proposition validates the relationship between the $q$-RECs:
the lower the $q$, the weaker the $q$-REC. However, the inverse of this implication is not true; see \cite[Example 1]{Hu2017Group} for a counter example. We provide the proof so as to make this paper self-contained, although the idea is similar to that of \cite[Proposition 5]{Hu2017Group}.
\begin{Proposition}\label{prop-REC}
Let $X\in \R^{m\times n}$, $a>0$, and $(s,t)$ be a pair of integers satisfying \eqref{s-t}. Suppose that $0<q_1\leq q_2\leq 1$ and that $X$ satisfies the $q_2$-\emph{REC}$(s,t,a)$. Then $X$ satisfies the $q_1$-\emph{REC}$(s,t,a)$.
\end{Proposition}
\begin{proof}
Associated with the $q$-REC$(s,t,a)$, we define the feasible set
\begin{equation}\label{eq-rec-fs}
C_q(s,a):=\{\delta\in \R^n:\|\delta_{J^c}\|_q^q\leq a\|\delta_{J}\|_q^q\ \text{for some}\  |J|\leq s\}.
\end{equation}
By Definition \ref{def-rec}, it remains to show that $C_{q_1}(s,a)\subseteq C_{q_2}(s,a)$. To this end, let $\delta\in C_{q_1}(s,a)$, and let $J_0$ denote the index set of the first $s$ largest coordinates in absolute value of $\delta$. By the assumption that $\delta\in C_{q_1}(s,a)$ and by the construction of $J_0$, one has $\|\delta_{J_0^c}\|_{q_1}^{q_1}\leq a\|\delta_{J_0}\|_{q_1}^{q_1}$. Then we obtain by Lemma \ref{lem-2} (with $q_2/q_1$ in place of $p$) that $\|\delta_{J_0^c}\|_{q_2}^{q_2}\leq a\|\delta_{J_0}\|_{q_2}^{q_2}$; consequently, $\delta\in C_{q_2}(s,a)$. Hence, it follows that $C_{q_1}(s,a)\subseteq C_{q_2}(s,a)$, and the proof is complete.
\end{proof}
\indent It is revealed from Proposition \ref{prop-REC} that the classical REC is a sufficient condition of the lower-order REC. In the sequel, we will further discuss some other types of regularity conditions: the sparse eigenvalues condition (SEC), the restricted isometry property (RIP), and the mutual incoherence property (MIP), which have been widely used in the literature of statistics and engineering, for ensuring the lower-order REC.

The SEC is a popular regularity condition required to guarantee the nice properties of sparse representation; see \cite{Bickel2009Simultaneous, Donoho2006Most, Zhang2012General} and references therein. For $\Delta \in \R^{n\times n}$ and $s\in \N$, the $s$-sparse minimal eigenvalue and $s$-sparse maximal eigenvalue of $\Delta$ are respectively defined by
\begin{equation}\label{sparse-eigen}
\sigma_{\min}(s,\Delta ):=\min_{\beta\in \R^n:1\leq \|\beta\|_0\leq s}\frac{\beta^\top \Delta \beta}{\beta^\top \beta},\quad
\sigma_{\max}(s,\Delta ):=\max_{\beta\in \R^n:1\leq \|\beta\|_0\leq s}\frac{\beta^\top \Delta\beta}{\beta^\top \beta}.
\end{equation}
The SEC was first introduced in \cite{Donoho2006Most} to show that the optimal solution of $(\text{CP}_{1,\epsilon})$ well approximates that of $(\text{CP}_{0,\epsilon})$ whenever $\sigma_{\min}(2s,\Gamma(X))>0$.

The RIP is another well-known regularity condition in the scenario of sparse learning, which was introduced by \cite{Candes2005Decoding} and has been widely used in the study of the oracle property and $\ell_2$ recovery bound for the high-dimensional regression model; see \cite{Bickel2009Simultaneous, Candes2006Stable, Recht2010Guaranteed} and references therein. Below, we recall the RIP-type notions from \cite{Candes2005Decoding}.
\begin{Definition}\cite{Candes2005Decoding}\label{def-RIP}
Let $X\in \R^{m \times n}$ and let $s,t\in \N$ be such that $s+t\leq n$.
\begin{enumerate}[{\rm (i)}]
  \item The $s$-restricted isometry constant of $X$, denoted by $\eta_s(X)$, is defined to be the smallest quantity such that, for
any $\beta\in \R^n$ and $J\subseteq \{1,\dots,n\}$ with $|J|\leq s$,
\begin{equation}\label{eq-RIC}
(1-\eta_s(X))\|\beta_J\|_2^2\leq \|X\beta_J\|_2^2\leq (1+\eta_s(X))\|\beta_J\|_2^2.
\end{equation}
  \item The $(s,t)$-restricted orthogonality constant of $X$, denoted by $\theta_{s,t}(X)$, is defined to be the smallest quantity such
that, for any $\beta\in \R^n$ and $J,T\subseteq \{1,\dots,n\}$ with $|J|\leq s$, $|T|\leq t$ and $J\cap T=\emptyset$,
\begin{equation}\label{eq-ROC}
|\langle X\beta_J,X\beta_{T}\rangle|\leq \theta_{s,t}(X)\|\beta_J\|_2\|\beta_{T}\|_2.
\end{equation}
\end{enumerate}
\end{Definition}

The MIP is also a well-known regularity condition in the scenario of sparse learning, which was introduced by \cite{Donoho2001Uncertainty} and has been used in \cite{Bickel2009Simultaneous, Cai2009Recovery, Donoho2006Most, Donoho2006Stable} and references therein. In the case when each diagonal element of the Gram matrix $\Gamma(X)$ is 1, $\theta_{1,1}(X)$ coincides with the mutual incoherence constant; see \cite{Donoho2001Uncertainty}.

The following lemmas are useful for establishing the relationship between the $q$-REC and other types of regularity conditions; in particular, Lemmas \ref{lem-RIP1} and \ref{lem-RIP2} are taken from \cite[Lemma 1.1]{Candes2005Decoding} and \cite[Lemma 3.1]{Geer2009On}, respectively.
\begin{Lemma}\label{lem-RIP1}
Let $X\in \R^{m\times n}$ and $s, t \in \N$ be such that $s+t\leq n$. Then
\[\theta_{s,t}(X)\leq \eta_{s+t}(X)\leq \theta_{s,t}(X)+\max\{\eta_s(X),\eta_{t}(X)\}.\]
\end{Lemma}
\begin{Lemma}\label{lem-RIP2}
Let $\alpha,\beta\in \R^n$ and $0<\tau<1$ be such that $-\langle\alpha,\beta\rangle\leq \tau\|\alpha\|_2^2$. Then $(1-\tau)\|\alpha\|_2\leq \|\alpha+\beta\|_2$.
\end{Lemma}

For the sake of simplicity, a partition structure and some notations are presented. For a vector $\delta\in \R^n$ and an index set $J\subseteq \{1,2,\dots,n\}$, we use ${\rm rank}(\delta_i;J^c)$ to denote the rank of the absolute value of $\delta_i$ in $J^c$ (in a decreasing order) and $J_k(\delta;t)$ to denote the index set of the $k$-th batch of the first $t$ largest coordinates in absolute value of $\delta$ in $J^c$. That is,
\begin{equation}\label{eq-Jk}
J_k(\delta;t):= \left\{i\in J^c:{\rm rank}(\delta_i;J^c)\in\{kt+1,\dots,(k+1)t\}\right\}\quad \mbox{for each } k\in \N.
\end{equation}
\begin{Lemma}\label{lem-SEC}
Let $X\in \R^{m\times n}$, $0<q\leq 1$, $a>0$, and $(s,t)$ be a pair of integers satisfying \eqref{s-t}. Then the following relations are true:
\begin{equation}\label{phi-1}
\phi_q(s,t,a,X)\geq \sqrt{\sigma_{\min}(s+t,\Gamma(X))}-a^{\frac1q}\left(\frac{s}t\right)^{\frac1q-\frac12}\sqrt{\sigma_{\max}(t,\Gamma(X))},
\end{equation}
\begin{equation}\label{phi-2}
\phi_q(s,t,a,X)\leq \sqrt{\sigma_{\max}(s+t,\Gamma(X))}+a^{\frac1q}\left(\frac{s}t\right)^{\frac1q-\frac12}\sqrt{\sigma_{\max}(t,\Gamma(X))}.
\end{equation}
\end{Lemma}
\begin{proof}
Fix $\delta \in C_q(s,a)$, as defined by \eqref{eq-rec-fs}. Then there exists $J\subseteq \{1,2,\dots,n\}$ such that
\begin{equation}\label{eq-new1}
|J|\leq s \quad \mbox{and}\quad \|\delta_{J^c}\|_q^q\leq a\|\delta_{J}\|_q^q.
\end{equation}
Write $r:=\lceil\frac{n-s}{t}\rceil$ (where $\lceil u\rceil$ denotes the largest integer not greater than $u$), $J_k:=J_k(\delta;t)$ (defined by \eqref{eq-Jk}) for each $k\in \N$ and $J_*:=J\cup J_0$.
Then it follows from \cite[Lemma 7]{Hu2017Group} and \eqref{eq-new1} that
\begin{equation}\label{I-14}
\sum_{k=1}^{r}\|\delta_{J_k}\|_2\leq t^{\frac12-\frac1q}\|\delta_{J^c}\|_q
\leq a^{\frac1q}t^{\frac12-\frac1q}\|\delta_{J}\|_q\leq a^{\frac1q}\left(\frac{s}t\right)^{\frac1q-\frac12}\|\delta_{J}\|_2
\end{equation}
(due to \eqref{I-1}). Noting by \eqref{eq-Jk} and \eqref{eq-new1} that $|J_*|\le s+t$ and $|J_k|\le t$ for each $k\in \N$, one has by \eqref{sparse-eigen} that
\[
\sqrt{\sigma_{\min}(s+t,\Gamma(X))}\|\delta_{J_*}\|_2 \le
\|X\delta_{J_*}\|_2 \leq \sqrt{\sigma_{\max}(s+t,\Gamma(X))}\|\delta_{J_*}\|_2,
\]
\[
\|X\delta_{J_k}\|_2
\leq \sqrt{\sigma_{\max}(t,\Gamma(X))}\|\delta_{J_k}\|_2\quad \text{for each }\ k\in \N.
\]
These, together with \eqref{I-14}, imply that
\begin{equation*}
\begin{aligned}
\|X\delta\|_2
&\geq \|X\delta_{J_*}\|_2-\sum_{k=1}^{r}\|X\delta_{J_k}\|_2\\
&\geq \left(\sqrt{\sigma_{\min}(s+t,\Gamma(X))}-a^{\frac1q}\left(\frac{s}t\right)^{\frac1q-\frac12}\sqrt{\sigma_{\max}(t,\Gamma(X))}\right)\|\delta_{J_*}\|_2.
\end{aligned}
\end{equation*}
Since $\delta$ and $J$ satisfying \eqref{eq-new1} are arbitrary, \eqref{phi-1} is shown to hold by \eqref{eq-rec} and the fact that $J_*=J\cup J(\delta;t)$. One can prove \eqref{phi-2} in a similar way, and thus, the details are omitted.
\end{proof}
\indent  The following proposition provides the sufficient conditions for the $q$-REC in terms of the SEC, RIP and MIP; see (a), (b) and (c) below respectively.
\begin{Proposition}\label{prop-RECnew}
Let $X\in \R^{m\times n}$, $0<q\leq 1$, $a>0$, and $(s,t)$ be a pair of integers satisfying \eqref{s-t}. Then $X$ satisfies the $q$-\emph{REC}$(s,t,a)$ provided that one of the following conditions:
\begin{enumerate}[{\rm (a)}]
  \item $\sigma_{\min}(s+t,\Gamma(X))>a\left(\frac{as}{t}\right)^{\frac2q-1}\sigma_{\max}(t,\Gamma(X)).$
  \item $\eta_{t}(X)+\theta_{s,t}(X)+a^{\frac12}\left(\frac{as}t\right)^{\frac1q-\frac12}\theta_{t,s+t}(X)<1$.
  \item each diagonal element of $\Gamma(X)$ is $1$ and $$\theta_{1,1}(X)<\left(\left(1+2a\left(\frac{as}t\right)^{\frac1q-1}\right)(s+t)\right)^{-1}.$$
\end{enumerate}
\end{Proposition}
\begin{proof}
It directly follows from Lemma \ref{lem-SEC} (cf. \eqref{phi-1}) that $X$ satisfies the $q$-REC$(s,t,a)$ provided that condition (a) holds.
Fix $\delta \in C_q(s,a)$, and let $J$, $r$, $J_k$ (for each $k\in \N$) and $J_*$ be defined, respectively, as in the beginning of the
proof of Lemma \ref{lem-SEC}. Then \eqref{I-14} follows directly and it
follows from \cite[Lemma 7]{Hu2017Group} and (17) that
\begin{equation}\label{MIP-1}
\|\delta_{J_*^c}\|_1=\sum_{k=1}^{r}\|\delta_{J_k}\|_1\leq t^{1-\frac1q}\|\delta_{J^c}\|_q
\leq a^{\frac1q}t^{1-\frac1q}\|\delta_{J}\|_q\leq a^{\frac1q}\left(\frac{s}t\right)^{\frac1q-1}\|\delta_{J}\|_1.
\end{equation}

Suppose that condition (b) is satisfied. By Definition 2 (cf. \eqref{eq-ROC}), one has that
\[
|\langle X\delta_{J_*}, X\delta_{J_*^c}\rangle|
\leq \sum_{k=1}^r|\langle X\delta_{J_*}, X\delta_{J_k}\rangle|\leq \theta_{t,s+t}(X) \|\delta_{J_*}\|_2 \sum_{k=1}^r\|\delta_{J_k}\|_2.
\]
Then it follows from \eqref{I-14} that
\begin{eqnarray}
|\langle X\delta_{J_*}, X\delta_{J_*^c}\rangle|
&\leq& a^{\frac1q}\left(\frac{s}t\right)^{\frac1q-\frac12}\theta_{t,s+t}(X)\|\delta_{J_*}\|_2\|\delta_{J}\|_2 \nonumber \\
&\leq& \frac{a^{\frac1q}\left(\frac{s}t\right)^{\frac1q-\frac12}\theta_{t,s+t}(X)}{1-\eta_{s+t}(X)}\|X\delta_{J_*}\|_2^2 \label{eq-RIP-1}
\end{eqnarray}
(by \eqref{eq-RIC}). Since $s\le t$ (by \eqref{s-t}), one has by Definition \ref{def-RIP}(i) that $\eta_s(X)\leq \eta_t(X)$, and then by Lemma \ref{lem-RIP1} that $\eta_{s+t}(X)\leq \theta_{s,t}(X)+\eta_t(X)$. Then it follows from (b) that
\begin{equation}\label{eq-RIP-2}
0<\frac{a^{\frac1q}\left(\frac{s}t\right)^{\frac1q-\frac12}\theta_{t,s+t}(X)}{1-\eta_{s+t}(X)}
\le \frac{a^{\frac1q}\left(\frac{s}t\right)^{\frac1q-\frac12}\theta_{t,s+t}(X)}{1-(\eta_{t}(X)+\theta_{s,t}(X))}<1.
\end{equation}
This, together with \eqref{eq-RIP-1}, shows that Lemma \ref{lem-RIP2} is applicable (with $X\delta_{J_*}$, $X\delta_{J_*^c}$, $\frac{a^{\frac1q}\left(\frac{s}t\right)^{\frac1q-\frac12}\theta_{t,s+t}(X)}{1-\eta_{s+t}(X)}$ in place of $\alpha$, $\beta$, $\tau$) to concluding that
\begin{eqnarray*}
\|X\delta\|_2^2
&\geq& \left(1-\frac{a^{\frac1q}\left(\frac{s}t\right)^{\frac1q-\frac12}\theta_{t,s+t}(X)}{1-\eta_{s+t}(X)}\right)^2\|X\delta_{J_*}\|_2^2\\
&\geq& (1-\eta_{s+t}(X))\left(1-\frac{a^{\frac1q}\left(\frac{s}t\right)^{\frac1q-\frac12}\theta_{t,s+t}(X)}{1-\eta_{s+t}(X)}\right)^2\|\delta_{J_*}\|_2^2
\end{eqnarray*}
(due to \eqref{eq-RIC}).
Since $\delta$ and $J$ satisfying \eqref{eq-new1} are arbitrary, we derive by \eqref{eq-rec} and \eqref{eq-RIP-2} that
\[
\phi_q(s,t,a,X)\geq \sqrt{1-\eta_{s+t}(X)}\left(1-\frac{a^{\frac1q}\left(\frac{s}t\right)^{\frac1q-\frac12}\theta_{t,s+t}(X)}{1-\eta_{s+t}(X)}\right)>0;
\]
consequently, $X$ satisfies the $q$-REC$(s,t,a)$.

Suppose that (c) is satisfied.
Then we have by \eqref{MIP-1} and Definition 2 (cf. \eqref{eq-ROC}) that
\begin{equation}\label{eq-MIP-new}
\begin{aligned}
\|X\delta\|_2^2&= \|X\delta_{J_*}\|_2^2+2\langle X\delta_{J_*}, X\delta_{J_*^c}\rangle+\|X\delta_{J_*^c}\|_2^2\\
&\geq \|X\delta_{J_*}\|_2^2-2|\langle X\delta_{J_*}, X\delta_{J_*^c}\rangle|\\
&\geq \|X\delta_{J_*}\|_2^2-2\theta_{1,1}(X)\|\delta_{J_*}\|_1\|\delta_{J_*^c}\|_1\\
&\geq \|X\delta_{J_*}\|_2^2-2a^{\frac1q}\left(\frac{s}t\right)^{\frac1q-1}\theta_{1,1}(X)\|\delta_{J_*}\|_1^2.
\end{aligned}
\end{equation}
Separating the diagonal and off-diagonal terms of the quadratic form $\delta_{J_*}^TX^TX\delta_{J_*}$, one has by \eqref{I-1} and (c) that
\begin{equation*}
\begin{aligned}
\|X\delta_{J_*}\|_2^2&= \sum_{i=1}^{n}(X^TX)_{i,i}(\delta_{J_*})_i(\delta_{J_*})_i+\sum_{j\ne k}^{n}(X^TX)_{j,k}(\delta_{J_*})_j(\delta_{J_*})_k\\
&= \|\delta_{J_*}\|_2^2+\sum_{j\ne k}^{n}\langle X_{\cdot j}(\delta_{J_*})_j, X_{\cdot k}(\delta_{J_*})_k \rangle\\
&\geq \|\delta_{J_*}\|_2^2-\theta_{1,1}(X)\|\delta_{J_*}\|_1^2\\
&\geq (1-(s+t)\theta_{1,1}(X))\|\delta_{J_*}\|_2^2.
\end{aligned}
\end{equation*}
Combining this inequality with \eqref{eq-MIP-new}, we get that
\begin{equation*}
\|X\delta\|_2^2\geq
\left(1-\left(1+2a^{\frac1q}\left(\frac{s}t\right)^{\frac1q-1}\right)(s+t)\theta_{1,1}(X)\right)\|\delta_{J_*}\|_2^2.
\end{equation*}
Since $\delta$ and $J$ satisfying \eqref{eq-new1} are arbitrary, we derive by \eqref{eq-rec} and (c) that
\[
\phi_q(s,t,a,X)\geq 1-\left(1+2a^{\frac1q}\left(\frac{s}t\right)^{\frac1q-1}\right)(s+t)\theta_{1,1}(X)>0;
\]
consequently, $X$ satisfies the $q$-REC$(s,t,a)$. The proof is complete.
\end{proof}
\begin{Remark}
It was established  in \cite[Lemma 4.1(ii)]{Bickel2009Simultaneous}, \cite[Corollary 7.1 and 3.1]{Geer2009On} and \cite[Assumption 5]{Bickel2009Simultaneous} that $X$ satisfies the classical \emph{REC} under one of the following conditions:
\begin{enumerate}[{\rm (a')}]
  \item $\sigma_{\min}(s+t,\Gamma(X))>\frac{s}{t}a^2\sigma_{\max}(t,\Gamma(X)).$
  \item $\eta_{t}(X)+\theta_{s,t}(X)+\left(\frac{s}t\right)^{\frac12}a\theta_{t,s+t}(X)<1$.
  \item each diagonal element of $\Gamma(X)$ is $1$ and $\theta_{1,1}(X)<((1+2a)(s+t))^{-1}.$
\end{enumerate}
Proposition \ref{prop-RECnew} extends the existing results to the general case when $0<q\le 1$ and partially improves them; in particular, each of conditions {\rm (a)}-{\rm (c)} in Proposition \ref{prop-RECnew} required for the $q$-\emph{REC} is less restrictive than the corresponding one of conditions {\rm (a')}-{\rm (c')} required for the classical \emph{REC} in the situation when $t>as$, which usually occurs in the high-dimensional scenario (see, e.g., \cite{Bickel2009Simultaneous, Candes2006Stable, Zhou2009Restricted}). Moreover, by Propositions \ref{prop-REC} and \ref{prop-RECnew},
we achieve that the $q$-\emph{REC}$(s,t,a)$ is satisfied provided that one of the following conditions:
\begin{enumerate}[{\rm (a$^\circ$)}]
  \item $\sigma_{\min}(s+t,\Gamma(X))>\min\left\{1,\left(\frac{as}{t}\right)^{\frac2q-2}\right\}\frac{s}{t}a^2\sigma_{\max}(t,\Gamma(X)).$
  \item $\eta_{t}(X)+\theta_{s,t}(X)+\min\left\{1, \left(\frac{as}t\right)^{\frac1q-1}\right\}\left(\frac{s}{t}\right)^{\frac12}a\theta_{t,s+t}(X)<1$.
  \item each diagonal element of $\Gamma(X)$ is $1$ and $$\theta_{1,1}(X)<\left(\left(1+2a\min\left\{1,\left(\frac{as}{t}\right)^{\frac{1}{q}-1}\right\}\right)(s+t)\right)^{-1}.$$
\end{enumerate}
\end{Remark}

\section{Recovery Bounds for Deterministic Design}
\noindent This section is devoted to establishing the $\ell_2$ recovery bounds for $(\text{CP}_{q,\epsilon})$ and $(\text{RP}_{q,\lambda})$ in the case that $X$ is deterministic.
Throughout this paper, we assume that the linear regression model \eqref{model} involves a Gaussian noise, i.e., $e\sim \mathscr{N}(0,\sigma^2\mathbb{I}_m)$, and adopt the following notations:
\begin{equation*}\label{eq-solution}
\mbox{let $\beta^*$ be a solution of \eqref{model}, $J:={\rm supp}(\beta^*)$, $s:=|J|$, and let $t\in \N$ satisfy \eqref{s-t}.}
\end{equation*}

The $\ell_2$ recovery bound of the $\ell_1$ regularization problem (i.e., Lasso estimator) was established in \cite{Bickel2009Simultaneous} under the assumption of the classical REC. The deduction of the $\ell_2$ recovery bound is based on an important property of the optimal solution. More precisely, let $\bar{\beta}_{1,\epsilon}$ and $\hat{\beta}_{1,\lambda}$ be the solutions of the $\ell_1$ minimization and the $\ell_1$ regularization problems, respectively. It was reported in \cite[Eq. (2.2)]{Candes2006Stable} and \cite[Corollary B.2]{Bickel2009Simultaneous} that the corresponding residuals satisfy the following dominant properties, with high probability,
\begin{equation*}
\|(\bar{\beta}_{1,\epsilon}-\beta^*)_{J^c}\|_1\leq \|(\bar{\beta}_{1,\epsilon}-\beta^*)_{J}\|_1
\end{equation*}
and
\begin{equation*}
\|(\hat{\beta}_{1,\lambda}-\beta^*)_{J^c}\|_1\leq 3\|(\hat{\beta}_{1,\lambda}-\beta^*)_{J}\|_1
\end{equation*}
for the $\ell_1$ minimization and the $\ell_1$ regularization problems, respectively.\\
\indent In the study of the $\ell_q$ minimization and the $\ell_q$ regularization problems, a natural question arises whether the residuals of solutions of $(\text{CP}_{q,\epsilon})$ or $(\text{RP}_{q,\lambda})$ satisfy such a dominant property on the support of the true underlying parameter of linear regression \eqref{model} with high probability. Below, we provide a positive answer for this question in Propositions \ref{prop-cone1} and \ref{prop-cone2}. To this end, we present some preliminary lemmas to measure the probabilities of random events related to the linear regression model \eqref{model}, in which Lemma \ref{lem-Xe-bound} is taken from \cite[Lemma C.1]{Zhou2009Restricted}.
\begin{Lemma}\label{lem-Xe-bound}
Let $0\leq \theta<1$ and $b\geq 0$. Suppose that
\begin{equation}\label{X-1}
\max\limits_{1\leq j\leq n}\|X_{\cdot j}\|_2\leq (1+\theta)\sqrt{m}.
\end{equation}
Then
\begin{equation*}
\mathbb{P}\left(\frac{\|X^{\top}e\|_\infty}{m}\geq \sigma(1+\theta)\sqrt{\frac{2(1+b)\log n}{m}}\right)\leq \left(n^b\sqrt{\pi\log n}\right)^{-1}.
\end{equation*}
\end{Lemma}
\begin{Lemma}\label{lem-e-bound}
Let $d\geq 5$. Then
\begin{equation*}
\mathbb{P}\left(\|e\|_2^2\geq dm\sigma^2\right)\leq \exp\left(-\frac{d-1}{4}m\right).
\end{equation*}
\end{Lemma}
\begin{proof}
Recall that $e=(e_1,\dots,e_m)^\top\sim \mathscr{N}(0,\sigma^2\mathbb{I}_m)$. Let $u_i:=\frac1{\sigma}e_i$ for $i=1,\dots,m$. Then one has that $u_1,\dots,u_m$ are i.i.d. Gaussian variables with $u_i\sim \mathscr{N}(0,1)$ for $i=1,\dots,m$. Let $u:=(u_1,\dots,u_m)^\top$. Clearly, $\|u\|_2^2=\frac1{\sigma^2}\|e\|_2^2$ is a chi-square random variable with $m$ degrees of freedom (see, e.g., \cite[Section 5.6]{Ross2009First}). Then it follows from standard tail bounds of chi-square random variable (see, e.g., \cite[Appendix I]{Raskutti2011Minimax}) that
\begin{equation*}
\mathbb{P}\left(\frac{\|u\|_2^2-m}{m}\geq d-1\right)\leq \exp\left(-\frac{d-1}{4}m\right)
\end{equation*}
(as $d\geq 5$). Consequently, we obtain that
\begin{equation*}
\mathbb{P}\left(\|e\|_2^2\geq dm\sigma^2\right)=\mathbb{P}\left(\|u\|_2^2\geq dm\right)\leq \exp\left(-\frac{d-1}{4}m\right).
\end{equation*}
The proof is complete.
\end{proof}
\indent Recall that $\beta^*$ satisfies the linear regression model \eqref{model}.
\begin{Lemma}\label{lem-RP}
Let $\hat{\beta}_{q,\lambda}$ be an optimal solution of  $(\emph{RP}_{q,\lambda})$. Then
\[
\frac{1}{2m}\|X\beta^*-X\hat{\beta}_{q,\lambda}\|_2^2\le \lambda\|\beta^*\|_q^q-\lambda\|\hat{\beta}_{q,\lambda}\|_q^q+
\frac{1}{m}\|\hat{\beta}_{q,\lambda}-\beta^*\|_1\|X^{\top}e\|_{\infty}.
\]
\end{Lemma}
\begin{proof}
Since $\hat{\beta}_{q,\lambda}$ is an optimal solution of $(\text{RP}_{q,\lambda})$, it follows that
\begin{equation*}
\frac{1}{2m}\|y-X\hat{\beta}_{q,\lambda}\|_2^2+\lambda\|\hat{\beta}_{q,\lambda}\|_q^q\leq \frac{1}{2m}\|y-X\beta^*\|_2^2+\lambda\|\beta^*\|_q^q.
\end{equation*}
This, together with \eqref{model}, yields that
\begin{equation*}
\begin{aligned}
\lambda\|\hat{\beta}_{q,\lambda}\|_q^q-\lambda\|\beta^*\|_q^q
&\leq \frac{1}{2m}\|y-X\beta^*\|_2^2-\frac{1}{2m}\|y-X\hat{\beta}_{q,\lambda}\|_2^2\\
&= \frac{1}{m}\left\langle X(\hat{\beta}_{q,\lambda}-\beta^*),e\right\rangle-\frac{1}{2m}\|X\beta^*-X\hat{\beta}_{q,\lambda}\|_2^2\\
&\leq \frac{1}{m}\|\hat{\beta}_{q,\lambda}-\beta^*\|_1\|X^{\top}e\|_{\infty}-\frac{1}{2m}\|X\beta^*-X\hat{\beta}_{q,\lambda}\|_2^2.
\end{aligned}
\end{equation*}
The proof is complete.
\end{proof}
\indent Below, we present some notations that are useful for the following discussion of the $\ell_2$ recovery bounds. Recall that $\beta^*$ is a solution of \eqref{model}. Throughout the remainder of this paper, let
\begin{equation}\label{eq-a-b}
a>1, \quad 0\le \theta<1, \quad b\ge0,
\end{equation}
unless otherwise specified, and let $r>0$ be such that
\begin{equation}\label{eq-r-1}
r\ge \|\beta^*\|_q.
\end{equation}
Let
\begin{equation}\label{rho}
\epsilon:=\sigma\sqrt{5m} \quad \mbox{and}\quad \rho:=\left(\frac{5\sigma^2}{2\lambda}+r^q\right)^{1/q},
\end{equation}
and select the regularization parameter in $(\text{RP}_{q,\lambda})$ as
\begin{equation}\label{lambda-1}
\lambda:=\max\left\{\frac{a+1}{a-1}\sigma(1+\theta) 2^{1-q}(1+r^q)^{\frac{1-q}{q}}\sqrt{\frac{2(1+b)\log n}{m}},\ \frac{5}{2}\sigma^2\right\}.
\end{equation}
Define the following two random events relative to linear regression model \eqref{model} by
\begin{equation}\label{A-ev}
\mathscr{A}:=\{e:\|e\|_2\leq \epsilon\}
\end{equation}
and
\begin{equation}\label{B-ev}
\mathscr{B}:=\left\{e:\frac{a+1}{(a-1)m}(2\rho)^{1-q}\|X^{\top}e\|_{\infty}\leq \lambda\right\}.
\end{equation}
\indent The following lemma estimates the probabilities of events $\mathscr{A}$ and $\mathscr{B}$.
\begin{Lemma}\label{lem-AB-pro}
The probability of event $\mathscr{A}$ satisfies
\begin{equation}\label{A-pro}
\mathbb{P}(\mathscr{A})\geq 1-\exp(-m).
\end{equation}
Moreover, suppose that \eqref{X-1} is satisfied. Then
\begin{equation}\label{B-pro}
\mathbb{P}(\mathscr{B}) \geq 1-\left(n^b\sqrt{\pi\log n}\right)^{-1},
\end{equation}
\begin{equation}\label{ACapB-pro}
\mathbb{P}(\mathscr{A}\cap \mathscr{B}) \geq 1-\exp(-m)-\left(n^b\sqrt{\pi\log n}\right)^{-1}.
\end{equation}
\end{Lemma}
\begin{proof}
By \eqref{rho} and \eqref{A-ev}, Lemma \ref{lem-e-bound} is applicable (with $d=5$) to showing that $\mathbb{P}(\mathscr{A}^c)\leq \exp(-m)$, that is, \eqref{A-pro} is proved. Then it remains to show \eqref{B-pro} and \eqref{ACapB-pro}. For this purpose, we have by \eqref{lambda-1} that
$\lambda\geq \frac{5}{2}\sigma^2$, and noting that $0<q\le 1$,
\begin{align*}
\lambda &\geq \frac{a+1}{a-1}\sigma(1+\theta) 2^{1-q}\left(\frac{5\sigma^2}{2\lambda}+r^q\right)^{\frac{1-q}{q}}\sqrt{\frac{2(1+b)\log n}{m}}\\
&= \frac{a+1}{a-1}\sigma(1+\theta)(2\rho)^{1-q}\sqrt{\frac{2(1+b)\log n}{m}}
\end{align*}
(due to \eqref{rho}). Then one has by \eqref{B-ev} that
\begin{eqnarray*}
\mathbb{P}(\mathscr{B}^c)
&\leq& \mathbb{P}\left(\frac{a+1}{(a-1)m}(2\rho)^{1-q}\|X^{\top}e\|_\infty\geq \frac{a+1}{a-1}\sigma(1+\theta)(2\rho)^{1-q}\sqrt{\frac{2(1+b)\log n}{m}}\right)\\
&=&
\mathbb{P}\left(\frac{\|X^{\top}e\|_\infty}{m}\geq \sigma(1+\theta)\sqrt{\frac{2(1+b)\log n}{m}}\right).
\end{eqnarray*}
Hence, by assumption \eqref{X-1}, Lemma \ref{lem-Xe-bound} is applicable to ensuring \eqref{B-pro}. Moreover, it follows from the elementary probability theory that
\begin{equation*}
\mathbb{P}(\mathscr{A}\cap \mathscr{B})\geq \mathbb{P}(\mathscr{A})-\mathbb{P}(\mathscr{B}^c)\geq 1-\exp(-m)-\left(n^b\sqrt{\pi\log n}\right)^{-1}.
\end{equation*}
The proof is complete.
\end{proof}
\indent We show in the following two propositions that the optimal solution $\hat{\beta}$ of the $\ell_q$ minimization problem $({\rm CP}_{q,\epsilon})$ or the $\ell_q$ regularization problem $({\rm RP}_{q,\lambda})$ satisfies the following dominant property on the support of the true underlying parameter of \eqref{model} with high probability:
\begin{equation}\label{dominant}
\|(\hat{\beta}-\beta^*)_{J^c}\|^q_q\leq c \|(\hat{\beta}-\beta^*)_{J}\|^q_q
\end{equation}
with $c=1$ or $c=a$, respectively.
\begin{Proposition}\label{prop-cone1}
Let  $\bar{\beta}_{q,\epsilon}$ be an optimal solution of $(\emph{CP}_{q,\epsilon})$ with $\epsilon$ given by \eqref{rho}. Then it holds under the event $\mathscr{A}$ that
\begin{equation}\label{I-16}
\|(\bar{\beta}_{q,\epsilon}-\beta^*)_{J^c}\|_q\leq \|(\bar{\beta}_{q,\epsilon}-\beta^*)_{J}\|_q.
\end{equation}
\end{Proposition}
\begin{proof}
Let $e\in \mathscr{A}$. Recall that $\beta^*$ satisfies the linear regression model \eqref{model}, one has that $\|y-X\beta^*\|_2=\|e\|_2\leq \epsilon$ (under the event $\mathscr{A}$), and so, $\beta^*$ is a feasible vector of $(\text{CP}_{q,\epsilon})$. Consequently, by the optimality of $\bar{\beta}_{q,\epsilon}$ for $(\text{CP}_{q,\epsilon})$, it follows that $\|\bar{\beta}_{q,\epsilon}\|_q\leq \|\beta^*\|_q$. Write $\delta:=\bar{\beta}_{q,\epsilon}-\beta^*$. Then we obtain that
\begin{equation}\label{eq-CP-solu}
\|\beta^*\|_q^q\geq \|\beta^*+\delta\|_q^q= \|\beta^*+\delta_{J}+\delta_{J^c}\|_q^q= \|\beta^*+\delta_{J}\|_q^q+\|\delta_{J^c}\|_q^q,
\end{equation}
where the last equality holds because $\beta^*_{J^c}=0$.
On the other hand, one has by \eqref{I-2} that
$\|\beta^*+\delta_{J}\|_q^q\geq \|\beta^*\|_q^q-\|\delta_{J}\|_q^q$.
This, together with \eqref{eq-CP-solu}, implies \eqref{I-16}. The proof is complete.
\end{proof}
\begin{Proposition}\label{prop-cone2}
Let $\hat{\beta}_{q,\lambda}$ be an optimal solution of $(\emph{RP}_{q,\lambda})$ with $\lambda$ given by \eqref{lambda-1}. Suppose that \eqref{X-1} is satisfied. Then
\begin{equation}\label{I-21}
\|\hat{\beta}_{q,\lambda}-\beta^*\|_1\leq (2\rho)^{1-q}\|\hat{\beta}_{q,\lambda}-\beta^*\|_q^q
\end{equation}
under the event $\mathscr{A}$, and
\begin{equation}\label{I-18}
\|(\hat{\beta}_{q,\lambda}-\beta^*)_{J^c}\|_q^q\leq a\|(\hat{\beta}_{q,\lambda}-\beta^*)_{J}\|_q^q
\end{equation}
under the event $\mathscr{A}\cap \mathscr{B}$.
\end{Proposition}
\begin{proof}
Let $e\in \mathscr{A}$. Since $\hat{\beta}_{q,\lambda}$ is an optimal solution of $(\text{RP}_{q,\lambda})$, one has that
\begin{equation*}
\frac{1}{2m}\|y-X\hat{\beta}_{q,\lambda}\|_2^2+\lambda\|\hat{\beta}_{q,\lambda}\|_q^q\leq \frac{1}{2m}\|y-X\beta^*\|_2^2+\lambda\|\beta^*\|_q^q.
\end{equation*}
Then, by \eqref{model} and \eqref{eq-r-1}, it follows that
\begin{equation*}
\|\hat{\beta}_{q,\lambda}\|_q^q\leq \frac{1}{2m\lambda}\|y-X\beta^*\|_2^2+\|\beta^*\|_q^q
\leq \frac{1}{2m\lambda}\|e\|_2^2+r^q
\leq \rho^q
\end{equation*}
(due to \eqref{rho} and $\eqref{A-ev}$). Write $\delta:=\hat{\beta}_{q,\lambda}-\beta^*$. Then, we obtain by \eqref{I-1} and \eqref{eq-r-1} that
\begin{equation*}
\|\delta\|_1\leq \|\hat{\beta}_{q,\lambda}\|_1+\|\beta^*\|_1\leq \|\hat{\beta}_{q,\lambda}\|_q+\|\beta^*\|_q\leq \rho+r<2\rho.
\end{equation*}
Consequently, noting that $0<q\le 1$, one sees that
$\frac{\|\delta\|_1}{2\rho}\leq \left(\frac{\|\delta\|_1}{2\rho}\right)^q$,
and then, by \eqref{I-1} that
\begin{equation}\label{eq-RP-so}
\|\delta\|_1\leq (2\rho)^{1-q}\|\delta\|_1^q\leq (2\rho)^{1-q}\|\delta\|_q^q.
\end{equation}
This shows that \eqref{I-21} is proved. Then it remains to claim \eqref{I-18}. To this end, noting that $\beta^*_{J^c}=0$,
we derive by Lemma \ref{lem-RP} that
\begin{equation*}
\begin{aligned}
-\frac{1}{m}\|\delta\|_1\|X^{\top}e\|_{\infty}
&\leq \lambda\|\beta^*\|_q^q-\lambda\|\beta^*+\delta\|_q^q \\
&= \lambda\|\beta_{J}^*\|_q^q-\lambda\|\beta_{J}^*+\delta_{J}\|_q^q-\lambda\|\delta_{J^c}\|_q^q \\
&\leq \lambda\left(\|\delta_{J}\|_q^q-\|\delta_{J^c}\|_q^q\right)
\end{aligned}
\end{equation*}
(by \eqref{I-2}). This, together with \eqref{eq-RP-so}, yields that
\begin{equation*}
\lambda\left(\|\delta_{J}\|_q^q-\|\delta_{J^c}\|_q^q\right)\ge -\frac{1}{m}(2\rho)^{1-q}\|\delta\|_q^q\|X^{\top}e\|_{\infty}.
\end{equation*}
Then, under the event $\mathscr{A}\cap \mathscr{B}$, we obtain by \eqref{B-ev} that
\[
(a+1)\left(\|\delta_{J}\|_q^q-\|\delta_{J^c}\|_q^q\right)\ge -(a-1)\|\delta\|_q^q= -(a-1)(\|\delta_{J}\|_q^q+\|\delta_{J^c}\|_q^q),
\]
which yields \eqref{I-18}. The proof is complete.
\end{proof}
\begin{Remark}\label{rem-domin}
By Lemma \ref{lem-AB-pro}, Propositions \ref{prop-cone1} and \ref{prop-cone2} show that \eqref{I-16} holds with probability at least $1-\exp(-m)$, and \eqref {I-18} holds with probability at least $1-\exp(-m)-\left(n^b\sqrt{\pi\log n}\right)^{-1}$ if \eqref{X-1} is satisfied, respectively.
\end{Remark}

By virtue of Lemma \ref{lem-AB-pro} and Proposition \ref{prop-cone1}, one of the main theorems of this section is as follows, in which we establish the $\ell_2$ recovery bound for the $\ell_q$ minimization problem $(\mbox{CP}_{q,\epsilon})$ under the $q$-REC. This theorem provides a unified framework to show that one can stably recover the underlying parameter with high probability via solving the $\ell_q$ minimization problem when the design matrix satisfies the weak $q$-REC.
\begin{Theorem}\label{thm-CP}
Let $\bar{\beta}_{q,\epsilon}$ be an optimal solution of $(\emph{CP}_{q,\epsilon})$ with $\epsilon$ given by \eqref{rho}. Suppose that $X$ satisfies the $q$-\emph{REC}$(s,t,1)$.  Then, with probability at least $1-\exp(-m)$, we have that
\begin{equation}\label{I-27}
\|\bar{\beta}_{q,\epsilon}-\beta^*\|_2^2\leq \frac{1+\left(\frac{s}t\right)^{\frac2q-1}}{\phi_q^2(s,t,1,X)}4\epsilon^2.
\end{equation}
\end{Theorem}
\begin{proof}
Write $\delta:=\bar{\beta}_{q,\epsilon}-\beta^*$, and let $J_*:=J\cup J_0(\delta;t)$ (defined by \eqref{eq-Jk}).
Fix $e\in \mathscr{A}$. Then it follows from \cite[Lemma 7]{Hu2017Group} and Proposition \ref{prop-cone1} that
\begin{equation*}
\|\delta_{J_*^c}\|_2^2\leq t^{1-\frac2q}\|\delta_{J^c}\|_q^2\leq t^{1-\frac2q}\|\delta_{J}\|_q^2\leq \left(\frac{s}t\right)^{\frac2q-1}\|\delta_J\|_2^2\leq \left(\frac{s}t\right)^{\frac2q-1}\|\delta_{J_*}\|_2^2
\end{equation*}
(by \eqref{I-1}), and so
\begin{equation}\label{I-28}
\|\delta\|_2^2=\|\delta_{J_*}\|_2^2+\|\delta_{J_*^c}\|_2^2 \leq \left(1+\left(\frac{s}t\right)^{\frac2q-1}\right)\|\delta_{J_*}\|_2^2.
\end{equation}
Recalling that $\beta^*$ satisfies the linear regression model \eqref{model}, we have that $\|y-X\beta^*\|_2=\|e\|_2\leq \epsilon$ (by \eqref{A-ev}), and then
\begin{equation}\label{I-29}
\|X\delta\|_2=\|X\bar{\beta}_{q,\epsilon}-X\beta^*\|_2\leq\|X\bar{\beta}_{q,\epsilon}-y\|_2+\|X\beta^*-y\|_2\leq 2\epsilon.
\end{equation}
On the other hand, Proposition \ref{prop-cone1} is applicable to concluding that \eqref{I-16} holds, which shows $\delta\in C_q(s,1)$ (cf. \eqref{eq-rec-fs}). Consequently, we obtain by the assumption of the $q$-REC$(s,t,1)$ that
\begin{equation*}
\|\delta_{J_*}\|_2\leq \frac{\|X\delta\|_2}{\phi_q(s,t,1,X)}.
\end{equation*}
This, together with \eqref{I-28} and \eqref{I-29}, implies that \eqref{I-27} holds under the event $\mathscr{A}$. Noting from Lemma \ref{lem-AB-pro} that $\mathbb{P}(\mathscr{A})\geq 1-\exp(-m)$, we obtain the conclusion. The proof is complete.
\end{proof}
\indent In the special case when the underlying data is noise-free, Theorem \ref{thm-CP} shows that $(\text{CP}_{q,\epsilon})$ can exactly predict the parameter for the deterministic linear regression with high probability under the lower-order REC. For the realistic scenario where the measurements are noisy-aware, Theorem \ref{thm-CP} illustrates the stable recovery capability of $(\text{CP}_{q,\epsilon})$ in the sense that its solution approaches to the true sparse parameter within a tolerance proportional to the noise level with high probability. Moreover, Theorem \ref{thm-CP} establishes the $\ell_2$ recovery bound $\|\bar{\beta}_{q,\epsilon}-\beta^*\|_2=O(\epsilon)$ under a weaker assumption than the RIP-type or MIP-type condition used in \cite{Dong2016Unified, Song2014Sparse}, respectively.

As a special case of Theorem \ref{thm-CP} when $q=1$, the following corollary presents the $\ell_2$ recovery bound of the $\ell_1$ minimization problem $(\text{CP}_{1,\epsilon})$ as
\begin{equation}\label{eq-rb-l1}
\|\bar{\beta}_{1,\epsilon}-\beta^*\|_2^2=O(\epsilon^2)
\end{equation}
under the classical REC. This result improves the ones in \cite{ Cai2009Recovery, Candes2006Stable} under a weaker assumption, in which the $\ell_2$ recovery bound \eqref{eq-rb-l1} was obtained under the RIP-type conditions.
\begin{Corollary}
Let $\bar{\beta}_{1,\epsilon}$ be an optimal solution of $(\emph{CP}_{1,\epsilon})$ with $\epsilon$ given by \eqref{rho}. Suppose that $X$ satisfies the $1$-\emph{REC}$(s,t,1)$.  Then, with probability at least $1-\exp(-m)$, we have that
\begin{equation*}
\|\bar{\beta}_{1,\epsilon}-\beta^*\|_2^2\leq \frac{1+\frac{s}t}{\phi_1^2(s,t,1,X)}4\epsilon^2.
\end{equation*}
\end{Corollary}

The other main theorem of this section is as follows, in which we exploit the statistical properties of the $\ell_q$ regularization problem $(\text{RP}_{q,\lambda})$ under the $q$-REC. The results include the estimation of prediction loss and recovery bound of parameter approximation, and also the oracle property, which provides an upper bound on the prediction loss plus the violation of false parameter estimation.
\begin{Theorem}\label{thm-RP}
Let $\hat{\beta}_{q,\lambda}$ be an optimal solution of $(\emph{RP}_{q,\lambda})$ with $\lambda$ given by \eqref{lambda-1}. Suppose that $X$ satisfies the $q$-\emph{REC}$(s,t,a)$ and that \eqref{X-1} is satisfied.  Then, with probability at least $1-\exp(-m)-\left(n^b\sqrt{\pi\log n}\right)^{-1}$, we have that
\begin{equation}\label{I-31}
\frac{1}{m}\|X\hat{\beta}_{q,\lambda}-X\beta^*\|_2^2\leq \left(\frac{2a\lambda}{(\phi_q(s,t,a,X)/\sqrt{m})^q}\right)^{\frac{2}{2-q}}s,
\end{equation}
\begin{equation}\label{I-32}
\frac{1}{2m}\|X\hat{\beta}_{q,\lambda}-X\beta^*\|_2^2+\lambda\|(\hat{\beta}_{q,\lambda})_{ J^c}\|_q^q\leq \left(\frac{2^{\frac{q}2}a\lambda}{(\phi_q(s,t,a,X)/\sqrt{m})^q}\right)^{\frac{2}{2-q}}s,
\end{equation}
\begin{equation}\label{I-33}
\|\hat{\beta}_{q,\lambda}-\beta^*\|_2^2\leq \left(1+a^{\frac2q}\left(\frac{s}{t}\right)^{\frac2q-1}\right)\left(\frac{2a\lambda}{(\phi_q(s,t,a,X)/\sqrt{m})^2}\right)^{\frac{2}{2-q}}s.
\end{equation}
\end{Theorem}
\begin{proof}
Write $\delta:=\hat{\beta}_{q,\lambda}-\beta^*$ and fix $e\in \mathscr{A}\cap \mathscr{B}$. Note by \eqref{I-21} and \eqref{B-ev} that
\begin{equation*}
\frac{1}{m}\|\delta\|_1\|X^{\top}e\|_\infty\leq \frac{a-1}{a+1}\lambda\|\delta\|_q^q.
\end{equation*}
This, together with Lemma \ref{lem-RP}, implies that
\begin{equation}\label{I-35}
\begin{aligned}
\frac{1}{2m}\|X\hat{\beta}_{q,\lambda}-X\beta^*\|_2^2
&\leq \lambda\|\beta^*\|_q^q-\lambda\|\hat{\beta}_{q,\lambda}\|_q^q+\frac{a-1}{a+1}\lambda\|\delta\|_q^q\\
&\le \lambda\|\delta_{J}\|_q^q-\lambda\|(\hat{\beta}_{q,\lambda})_{J^c}\|_q^q+\frac{a-1}{a+1}\lambda\|\delta\|_q^q
\end{aligned}
\end{equation}
(noting that $\beta^*_{J^c}=0$ and by \eqref{I-2}). Let $J_*:=J\cup J_0(\delta;t)$. One has by \eqref{I-18} and \eqref{I-1} that
\begin{equation*}
\lambda\|\delta_{J}\|_q^q+\frac{a-1}{a+1}\lambda\|\delta\|_q^q  \leq a\lambda\|\delta_{J}\|_q^q  \leq a\lambda s^{1-\frac{q}2}\|\delta_{J}\|_2^q,
\end{equation*}
and by the assumption of the $q$-REC$(s,t,a)$ that
\begin{equation*}
\|\delta_{J}\|_2\leq \|\delta_{J^*}\|_2\leq \frac{\|X\delta\|_2}{\phi_q(s,t,a,X)}.
\end{equation*}
These two inequalities, together with \eqref{I-35}, imply that
\begin{equation*}
\frac{1}{2m}\|X\hat{\beta}_{q,\lambda}-X\beta^*\|_2^2+\lambda\|(\hat{\beta}_{q,\lambda})_{J^c}\|_q^q \leq \frac{a\lambda s^{1-\frac{q}2}}{\phi_q^q(s,t,a,X)}\|X\hat{\beta}_{q,\lambda}-X\beta^*\|_2^q.
\end{equation*}
This yields that
\begin{equation}\label{eq-RB-RP-2}
\mbox{\eqref{I-31} and \eqref{I-32} hold under the event $\mathscr{A}\cap \mathscr{B}$.}
\end{equation}
Furthermore, it follows from \cite[Lemma 7]{Hu2017Group} that
\begin{equation*}
\|\delta_{J_*^c}\|_2^2\leq t^{1-\frac2q}\|\delta_{J^c}\|_q^2\leq a^{\frac2q}t^{1-\frac2q}\|\delta_{J}\|_q^2\leq a^{\frac2q}\left(\frac{s}{t}\right)^{\frac2q-1}\|\delta_{J}\|_2^2.
\end{equation*}
(by \eqref{I-18} and \eqref{I-1}). By the assumption of the $q$-REC$(s,t,a)$, one has by \eqref{I-31} that
\begin{equation*}
\|\delta_{J_*}\|_2^2\leq \frac{\|X\delta\|_2^2}{\phi_q^2(s,t,a,X)}\leq \left(\frac{2a\lambda}{(\phi_q(s,t,a,X)/\sqrt{m})^2}\right)^{\frac{2}{2-q}}s.
\end{equation*}
Hence we obtain that
\begin{equation*}
\begin{aligned}
\|\hat{\beta}_{q,\lambda}-\beta^*\|_2^2
&= \|\delta_{J_*}\|_2^2+\|\delta_{J_*^c}\|_2^2\leq \left(1+a^{\frac2q}\left(\frac{s}{t}\right)^{\frac2q-1}\right)\|\delta_{J_*}\|_2^2 \\
&\leq \left(1+a^{\frac2q}\left(\frac{s}{t}\right)^{\frac2q-1}\right)
\left(\frac{2a\lambda}{(\phi_q(s,t,a,X)/\sqrt{m})^2}\right)^{\frac{2}{2-q}}s. \end{aligned}
\end{equation*}
This shows that
\begin{equation}\label{eq-RB-RP-3}
\mbox{\eqref{I-33} holds under the event $\mathscr{A}\cap \mathscr{B}$.}
\end{equation}
By assumption \eqref{X-1}, Lemma \ref{lem-AB-pro} is applicable to concluding that
\[
\mathbb{P}(\mathscr{A}\cap \mathscr{B}) \geq 1-\exp(-m)-\left(n^b\sqrt{\pi\log n}\right)^{-1}.
\]
This, together with \eqref{eq-RB-RP-2} and \eqref{eq-RB-RP-3}, yields that \eqref{I-31}-\eqref{I-33} hold with probability at least $1-\exp(-m)
-\left(n^b\sqrt{\pi\log n}\right)^{-1}$. The proof is complete.
\end{proof}
\begin{Remark}\label{rmk-thm2}
{\rm (i)} It is worth noting that each of the estimations provided in Theorem \ref{thm-RP} (cf. \eqref{I-31}-\eqref{I-33}) involves the term $\phi_q(s,t,a,X)/\sqrt{m}$ in the denominator, which scales as a constant if $X$ has i.i.d. Gaussian entries; see Remark \ref{rmk-sqrtm}(ii).

{\rm (ii)} Theorem \ref{thm-RP} provides a unified framework of the statistical properties of the $\ell_q$ regularization problem under the weak $q$-{\rm REC} that is one of the weakest regularity conditions in the literature, in which each of the obtained estimations depends on the noise amplitude and sample size.
In particular, for the regularization parameter scaling as $\lambda\asymp \max\left(\sigma\sqrt{\frac{\log n}{m}},\ \sigma^2\right)$ (cf. \eqref{lambda-1}), Theorem \ref{thm-RP}  indicates the prediction loss and the $\ell_2$ recovery bound for $({\rm RP}_{q,\lambda})$ scale as
\[
\frac{1}{m}\|X\hat{\beta}_{q,\lambda}-X\beta^*\|_2^2=O\left(\left(\sigma^2 \frac{\log n}{m}\right)^{\frac{1}{2-q}}s\right),
\]
and
\begin{equation}\label{eq-l2bound}
\|\hat{\beta}_{q,\lambda}-\beta^*\|_2^2=O\left(\left(\sigma^2 \frac{\log n}{m}\right)^{\frac{1}{2-q}}s\right).
\end{equation}
Though the rate \eqref{eq-l2bound} in the case $q<1$ is not as good as that of \emph{Lasso}, the required regularity condition is substantially weaker. Specifically, for some applications that the $q$-\emph{REC} is satisfied but not the classical \emph{REC} (e.g., Example \ref{exam-REC} below), the recovery bound for \emph{Lasso} may violate and lead to a bad estimation while the $\ell_q$ regularization model still works and produces a comprehensive estimation.

{\rm (iii)} It was shown in \cite{Zhang2012General} that the global solution of the {\rm FCP} sparse linear regression, including the {\rm SCAD} and {\rm MCP} as special cases, has an $\ell_2$ recovery bound $O(\lambda^2 s)$ under the {\rm SEC}. Though the recovery bounds are slightly better than \eqref{eq-l2bound}, the condition required is substantially stronger than the $q$-{\rm REC}. In \cite{Zhang2012General}, the authors also established the oracle property for the $\ell_0$ regularization method under the {\rm SEC}; while its $\ell_2$ recovery bound cannot be guaranteed in their work. We shall see in section 5 that the $\ell_q$ regularization method performs better in parameter estimation than either the {\rm SCAD/MCP} or the $\ell_0$ regularization method via several numerical experiments.

{\rm (iv)} Mazumder et al. \cite{Mazumder2017The, Mazumder2017Subset} considered the following $\ell_0$ optimization problems
\begin{equation}\label{mazu1}
\min\, \|\beta\|_0,\quad {\rm s.t.}\quad \left\|\frac{1}{m}X^\top(y-X\beta)\right\|_\infty\leq \epsilon,
\end{equation}
and
\begin{equation}\label{mazu2}
\min\, \frac{1}{2m}\|y-X\beta\|_2^2+\lambda\|\beta\|_p,\quad {\rm s.t.}\quad \|\beta\|_0\leq s\ (p=1\ {\rm or}\ 2),
\end{equation}
respectively.
It was shown in \cite{Mazumder2017The} that the $\ell_2$ recovery bound for problem \eqref{mazu1} scales as $O(\epsilon^2)$ with high probability, which is similar to \eqref{eq-rb-l1}, under the \emph{SEC}-type condition.
While its assumed regularity condition is stronger than the $q$-{\rm REC}; see Proposition \ref{prop-RECnew}.
In \cite{Mazumder2017Subset}, the authors established the prediction loss for problem \eqref{mazu2}, i.e., $O(\sigma\sqrt{\log n}\|\beta^*\|_1)$ when $p=1$, and $O(\sigma\sqrt{s\log n}\|\beta^*\|_2)$ when $p=2$. However, the $\ell_2$ recovery bound was not obtained yet therein.
\end{Remark}

\begin{Remark}
Recently, some works concerned the statistical property for the local minimum of some nonconvex regularization problems; see \cite{Liu2017Folded, Loh2015Regularized}.

{\rm (i)} Loh and Wainwright \cite{Loh2015Regularized} studied the $\ell_2$ recovery bound for the local minimum of a general regularization problem:
\begin{equation}\label{JMLR-func}
\min\, \mathcal{L}_m(\beta;X)+\sum_{j=1}^n\rho_\lambda(\beta_j),
\end{equation}
where $\mathcal{L}_m:\R^n\times \R^{m\times n}\to \R$ is the loss function, and $\rho_\lambda:\R\to \R$ is the (possibly nonconvex) penalty function. In \cite{Loh2015Regularized}, the penalty function $\rho_\lambda$ is assumed to satisfy the following assumptions:

{\rm (a)} $\rho_\lambda(0)=0$ and is symmetric around zero;

{\rm (b)} $\rho_\lambda$ is nondecreasing on $\R_+$;

{\rm (c)} For $t>0$, the function $t\mapsto \frac{\rho_\lambda(t)}{t}$ is nonincreasing in $t$;

{\rm (d)} $\rho_\lambda$ is differentiable for each $t\neq 0$ and subdifferentiable at $t=0$, with $\lim\limits_{t\to 0^+}\rho_\lambda^{'}(t)=\lambda L$;

{\rm (e)} There exists $\mu>0$ such that $\rho_{\lambda,\mu}(t):=\rho_\lambda(t)+\frac{\mu}{2}t^2$ is convex.\\
Loh and Wainwright established in \cite[Theorem 1]{Loh2015Regularized} the $\ell_2$ recovery bound for the critical point satisfying the first-order necessary condition of \eqref{JMLR-func} under the restricted strong convex condition, which is a variant of the classical {\rm REC}.

The $\ell_q$ norm can be reformulated as the penalty function $\rho_\lambda(\beta_j):=\lambda|\beta_j|^q$, however, it does not satisfy assumptions {\rm (d)} or {\rm (e)}; in particular, assumption {\rm (e)} plays a key role in the establishment of oracle property and $\ell_2$ recovery bound for the local minimum. Therefore, the result in \cite{Loh2015Regularized} cannot be directly applied to the $\ell_q$ regularization problem, and the oracle property for the general local minimum of the $\ell_q$ regularization problem is still an open question at this moment.

{\rm (ii)} Liu et al. \cite{Liu2017Folded} studied the statistical property of the {\rm FCP} sparse linear regression and presented the oracle property and $\ell_2$ recovery bound for the certain local minimum, which satisfies a subspace second-order necessary condition and lies in the level set of the {\rm FCP} regularized function at the true solution, under the {\rm SEC}. Although the $\ell_q$ regularizer is beyond the {\rm FCP}, our established Theorem \ref{thm-RP} provides a theoretical result similar to \cite{Liu2017Folded} in the sense that the oracle property and $\ell_2$ recovery bound are shown for the local minimum within the level set of the $\ell_q$ regularized function at the true solution.
\end{Remark}

\begin{Example}\label{exam-REC}
Consider the linear regression problem \eqref{model}, where
\[
X:=\begin{pmatrix}
    2 & 3 & 1 \\
    2 & 1 & 3
  \end{pmatrix},\quad
\beta^*:=(1,0,0)^\top,\quad
e\sim \mathscr{N}(0,0.01).
\]
It was validated in \cite[Example 1]{Hu2017Group} that the matrix $X$ satisfies $1/2$-\emph{REC}$(1,1,1)$ but not the classical \emph{REC}$(1,1,1)$; hence the recovery bound for the $\ell_{1/2}$ regularization problem is satisfied but may not for \emph{Lasso}. To show the performance of the $\ell_{1/2}$ regularization problem and \emph{Lasso} in this case, for each regularization parameter $\lambda$ varying from $10^{-8}$ to $1$, we randomly generate the Gaussian noise 500 times and calculate the estimated errors $\|\hat{\beta}_{q,\lambda}-\beta^*\|_2^2$ for the $\ell_{1/2}$ regularization problem and \emph{Lasso}, respectively. We employ \emph{FISTA} \cite{Beck2009Fast} and the filled function method \cite{Ge90} to find the global optimal solution of \emph{Lasso} and the $\ell_{1/2}$ regularization problem, respectively. The results are illustrated in Figure \ref{f-REC}, in which the error bars represent the 95\% confidence intervals and the curves of recovery bounds stand for the terms in the right-hand side of \eqref{I-33} (cf. \cite[Example 2]{Hu2017Group}) and \eqref{eq-l2bound}, respectively. It is observed from Figure \ref{f-REC}(a) that the recovery bound \eqref{I-33} is satisfied with high probability for most of $\lambda$'s and tight when $\lambda\thickapprox \frac12$ for the $\ell_{1/2}$ regularization problem. Figure \ref{f-REC}(b) shows that the estimated error \eqref{eq-l2bound} for \emph{Lasso} is not satisfied when $\lambda$ is small because the classical \emph{REC} violates. Moreover, the solutions of \emph{Lasso} are always equal-contributed among 3 components that leads to the failure approach to a sparse solution.
%
\begin{figure}[h]
\centering
\mbox{ \subfigure[The $\ell_{1/2}$ regularization problem.]{\includegraphics[width=6.5cm]{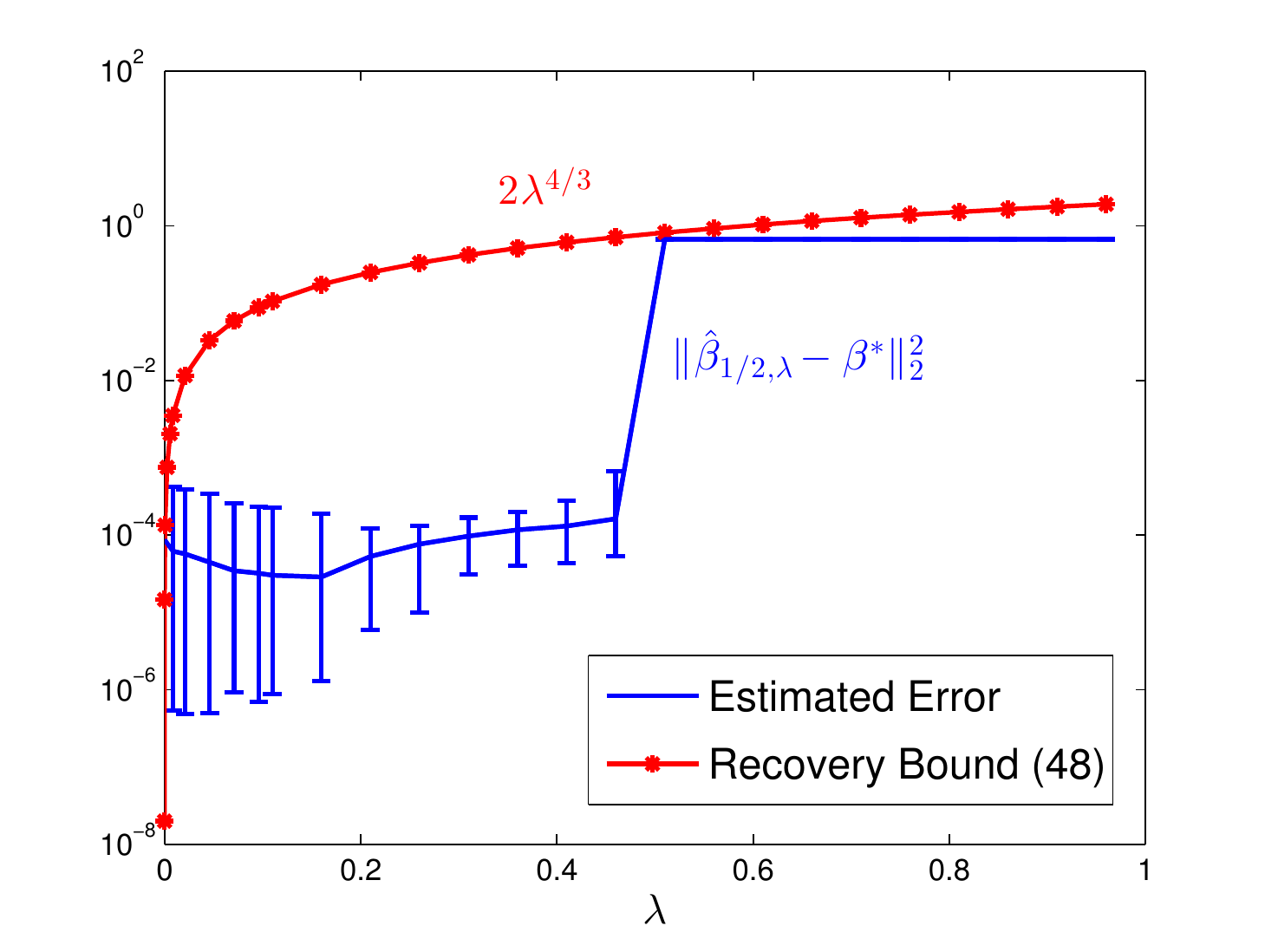}} \quad
\subfigure[Lasso.]{\includegraphics[width=6.5cm]{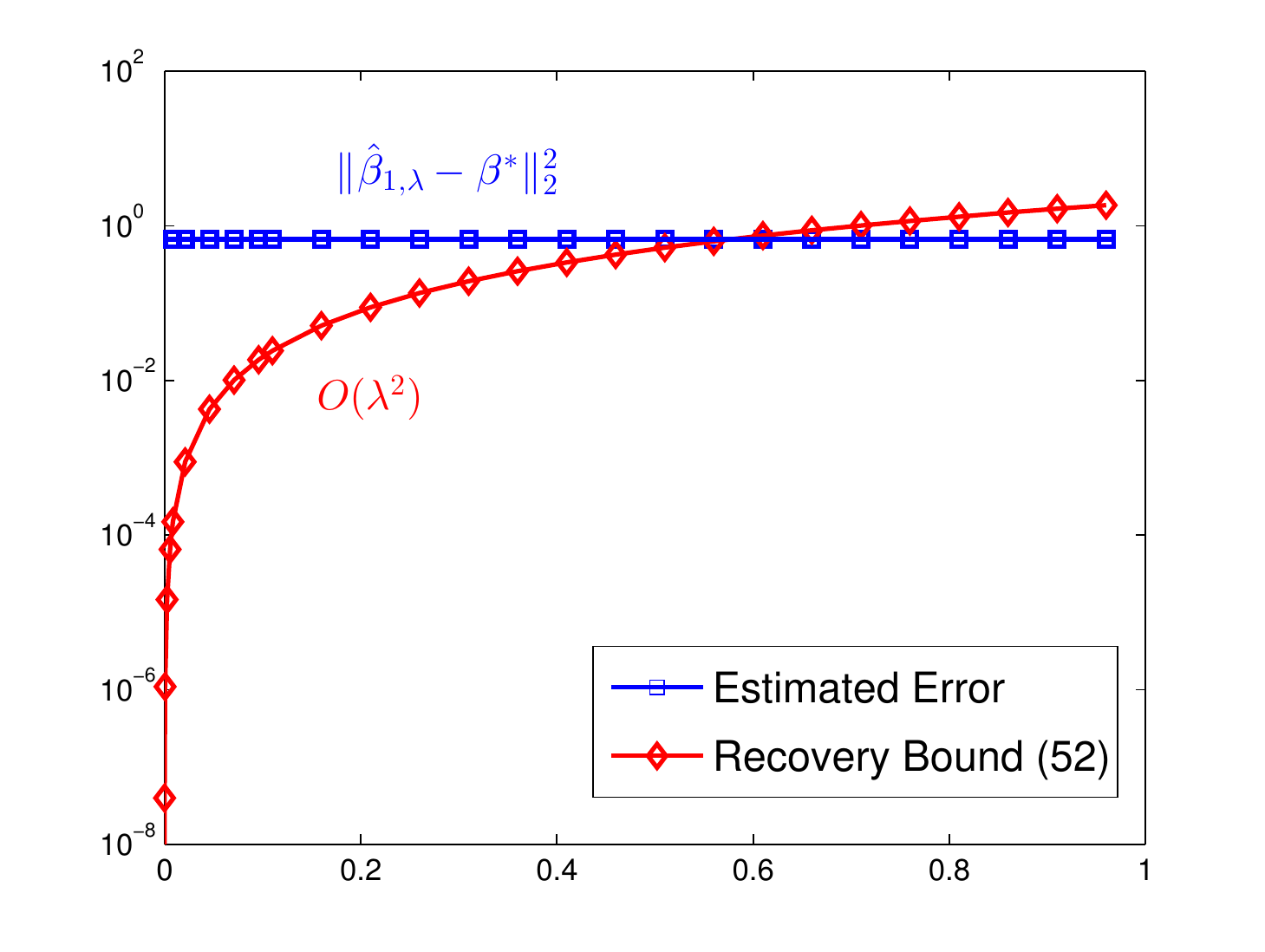}} }
\caption{The illustration of recovery bounds and estimated errors.}\label{f-REC}
\label{fig-LR}
\end{figure}
\end{Example}

As an application of Theorem \ref{thm-RP} to the case when $q=1$, the following corollary presents the statistical properties of the $\ell_1$ regularization problem under the classical REC, which covers \cite[Theorem 7.2]{Bickel2009Simultaneous} as a special case when $a=3$, $\theta=0$ and $b=0$. The same $\ell_2$ recovery bound rate $O(\sigma^2s\log n/m)$ was reported in \cite{Zhang2008The} under the sparse Riesz condition, which is comparable with the classical REC; while the same oracle inequality rate $O(\sigma^2s\log n/m)$ was established in \cite{Geer2009On} under the compatibility condition, which is slightly weaker than the classical REC but cannot guarantee the $\ell_2$ recovery bound.
\begin{Corollary}\label{corol-RP}
Let $\hat{\beta}_{1,\lambda}$ be an optimal solution of $(\emph{RP}_{1,\lambda})$ with
\[
\lambda=2\sigma(1+\theta)\sqrt{\frac{2(1+b)\log n}{m}}.
\]
Suppose that $X$ satisfies the $1$-\emph{REC}$(s,t,3)$ and that \eqref{X-1} is satisfied.
Then, with probability at least $1-\left(n^b\sqrt{\pi\log n}\right)^{-1}$, we have that
\begin{equation*}
\frac{1}{m}\|X\hat{\beta}_{1,\lambda}-X\beta^*\|_2^2\leq
\frac{288(1+b)(1+\theta)^2}{\phi_1^2(s,t,3,X)/m}\sigma^2 s\frac{\log n}{m},
\end{equation*}
\begin{equation*}
\frac{1}{2m}\|X\hat{\beta}_{1,\lambda}-X\beta^*\|_2^2+\lambda\|(\hat{\beta}_{1,\lambda})_{ J^c}\|_1\leq \frac{144(1+b)(1+\theta)^2}{\phi_1^2(s,t,3,X)/m}\sigma^2 s\frac{\log n}{m},
\end{equation*}
\begin{equation*}
\|\hat{\beta}_{1,\lambda}-\beta^*\|_2^2\leq 
\frac{288(1+b)(1+\theta)^2\left(1+9\frac{s}t\right)}{\phi_1^4(s,t,3,X)/m^2}\sigma^2 s\frac{\log n}{m}.
\end{equation*}
\end{Corollary}

\section{Recovery Bounds for Random Design}
\noindent In practical applications, it is a more realistic scenario that the design matrix $X$ is random.
In this section, we consider this situation and present the $\ell_2$ recovery bounds for $(\text{CP}_{q,\epsilon})$ and $(\text{RP}_{q,\lambda})$
by virtue of the results obtained in the preceding section. In particular, throughout this section, we shall assume that the linear regression model \eqref{model} involves a Gaussian noise, i.e., $e\sim \mathscr{N}(0,\sigma^2\mathbb{I}_m)$, and
\begin{equation*}
X\in \R^{m\times n}\  \text{is a Gaussian random design with i.i.d.}\  \mathscr{N}(0,\Sigma)\ \text{rows},
\end{equation*}
that is, $X_{1\cdot},\dots,X_{m\cdot}$ are i.i.d. random vectors with each $X_{i\cdot}\sim \mathscr{N}(0,\Sigma)$. Recall that $a$, $\theta$, and $b$ are given by \eqref{eq-a-b},
and let $(s,t)$ be a pair of integers satisfying \eqref{s-t}. \\
\indent To study the statistical properties of $(\text{CP}_{q,\epsilon})$ and $(\text{RP}_{q,\lambda})$ with a random design $X$, we first provide some sufficient condition for the $q$-REC of $X$ in terms of the population covariance matrix $\Sigma$. For this purpose, we use $\Sigma^{\frac{1}{2}}$ to denote the square root of $\Sigma$ and $\zeta(\Sigma):=\max_{1\leq j\leq n}\Sigma_{j,j}$ to denote the maximal variance.
Let $a>0$, and two random events related to the linear regression model \eqref{model} with $X$ being a Gaussian random design are defined as follows
\begin{equation}\label{C-ev}
\mathscr{C}_a:=\left\{\phi_q(s,t,a,X)>\frac{\sqrt{m}}{2}\phi_q(s,t,a,\Sigma^{\frac{1}{2}})\right\},
\end{equation}
and
\begin{equation}\label{D-ev}
\mathscr{D}:=\left\{\max_{1\leq j\leq n}\|X_{\cdot j}\|_2\leq (1+\theta)\sqrt{m}\right\}.
\end{equation}
\indent The following lemma is taken from \cite[Supplementary, Lemma 6]{Agarwal2012Fast}, which is useful for providing a sufficient condition for the $q$-REC of $X$.
\begin{Lemma}\label{lem-REC-randomX}
There exist universal positive constants $(c_1,c_2)$ (independent of $m,n,\Sigma$) such that it holds with probability at least $1-\exp(-c_2m)$ that, for each $\delta\in \R^n$
\begin{equation}\label{X-Sigma}
\frac{\|X\delta\|_2^2}{m}\geq \frac{1}{2}\|\Sigma^{\frac{1}{2}}\delta\|_2^2-c_1\zeta(\Sigma)\frac{\log n}{m}\|\delta\|_1^2.
\end{equation}
\end{Lemma}

The following lemma calculates the probabilities of events $\mathscr{C}_c$ and $\mathscr{D}$, which is crucial for establishing the $\ell_2$ recovery bounds of $(\text{CP}_{q,\epsilon})$ and $(\text{RP}_{q,\lambda})$ with a random design $X$. In particular, part (i) of this lemma shows that the Gaussian random design $X$ satisfies the $q$-REC with high probability as long as the sample size $m$ is sufficiently large and the square root of its population covariance matrix $\Sigma^{\frac{1}{2}}$ satisfies the $q$-REC; part (ii) of this lemma presents that each column of the Gaussian random design $X$ has an Euclidean norm scaling as $\sqrt{m}$ with an overwhelming probability.
\begin{Lemma}\label{lem-CD-pro}
{\rm (i)} Let $a>0$. Suppose that $\Sigma^{\frac{1}{2}}$ satisfies the $q$-\emph{REC}$(s,t,a)$. Then, there exist universal positive constants $(c_1,c_2)$ (independent of $m,n,q,$ $s,t,a,\Sigma$) such that, if
\begin{equation}\label{REC-sample}
m>\frac{c_1\zeta(\Sigma)}{\phi_q^2(s,t,a,\Sigma^{\frac{1}{2}})}\left(\sqrt{s+t}+a\sqrt{s}\left(\frac{as}t\right)^{\frac1q-1}\right)^2\log n,
\end{equation}
then
\begin{equation}\label{C-pro}
\mathbb{P}(\mathscr{C}_a)\geq 1-\exp(-c_2m).
\end{equation}
{\rm (ii)} Suppose that $\Sigma_{j,j}=1$ for all $j=1,\dots,n$. Then, there exist universal positive constants $(c_3,c_4)$ and $\tau\geq 1$ (independent of $m,n,\theta,\Sigma$) such that, if
\begin{equation}\label{X-theta-sample}
m>\frac{c_3\tau^4}{\theta^2}\log n,
\end{equation}
then
\begin{equation}\label{D-pro}
\mathbb{P}(\mathscr{D})\geq 1-2\exp (-c_4\theta^2m/\tau^4).
\end{equation}
\end{Lemma}
\begin{proof}
{\rm (i)} We first claim that
\begin{equation}\label{X-Sigma-1}
\phi_q(s,t,a,X)>\frac{\sqrt{m}}{2}\phi_q(s,t,a,\Sigma^{\frac{1}{2}}),
\end{equation}
whenever \eqref{X-Sigma} holds for each $\delta\in \R^n$. To this end, we suppose that \eqref{X-Sigma} is satisfied for each $\delta\in \R^n$.
Fix $\delta \in C_q(s,a)$, and let $J$, $r$, $J_k$ (for each $k\in \N$) and $J_*$ be defined, respectively, as in the beginning of the proof of Lemma \ref{lem-SEC}. Then \eqref{MIP-1} follows directly, and one has that
\begin{equation}\label{I-39}
\begin{aligned}
\|\delta\|_1
&=\|\delta_{J_*}\|_1+\|\delta_{J_*^c}\|_1\\
&\leq \sqrt{s+t}\|\delta_{J_*}\|_2+a\sqrt{s}\left(\frac{as}t\right)^{\frac1q-1}\|\delta_{J}\|_2\\
&\leq \left(\sqrt{s+t}+a\sqrt{s}\left(\frac{as}t\right)^{\frac1q-1}\right)\|\delta_{J_*}\|_2.
\end{aligned}
\end{equation}
By the assumption that $\Sigma^{\frac12}$ satisfies the $q$-REC$(s,t,a)$, it follows that
\begin{equation*}
\|\Sigma^{\frac{1}{2}}\delta\|_2^2\geq \phi_q^2(s,t,a,\Sigma^{\frac{1}{2}})\|\delta_{J_*}\|_2^2.
\end{equation*}
Substituting this inequality and \eqref{I-39} into \eqref{X-Sigma} yields
\begin{equation*}
\frac{\|X\delta\|_2^2}{m}\geq \left(\frac{1}{2}\phi_q^2(s,t,a,\Sigma^{\frac{1}{2}})-c_1\zeta(\Sigma)\left(\sqrt{s+t}+a\sqrt{s}\left(\frac{as}t\right)^{\frac1q-1}\right)^2\frac{\log n}{m}\right)\|\delta_{J_*}\|_2^2.
\end{equation*}
This, together with \eqref{REC-sample}, shows that
\begin{equation*}
\frac{\|X\delta\|_2^2}{m}\geq \frac{1}{4}\phi_q^2(s,t,a,\Sigma^{\frac{1}{2}})\|\delta_{J_*}\|_2^2.
\end{equation*}
Since $\delta$ and $J$ satisfying \eqref{eq-new1} are arbitrary, we derive by \eqref{eq-rec} that \eqref{X-Sigma-1} holds, as desired. Then, Lemma \ref{lem-REC-randomX} is applicable to concluding \eqref{C-pro}.\\
{\rm (ii)} Noting by the assumption that $\Sigma_{j,j}=1$ for all $j=1,\dots,n$, \cite[Theorem 1.6]{Zhou2009Restricted} is applicable to showing that there exist universal positive constants $(c_1,c_2)$ and $\tau\geq 1$ such that
\begin{equation*}
\mathbb{P}\left(\cap_{j=1}^n\left\{(1-\theta)\sqrt{m}\leq \|X_{\cdot j}\|_2\leq (1+\theta)\sqrt{m}\right\}\right)
\geq 1-2\exp (-c_2\theta^2m/\tau^4),
\end{equation*}
whenever $m$ satisfies \eqref{X-theta-sample}.
Then it immediately follows from \eqref{D-ev} that
\begin{equation*}
\begin{aligned}
\mathbb{P}(\mathscr{D})
&= \mathbb{P}(\cap_{j=1}^n\{\|X_{\cdot j}\|_2\leq (1+\theta)\sqrt{m}\})\\
&\geq 1-2\exp (-c_2\theta^2m/\tau^4),
\end{aligned}
\end{equation*}
that is, \eqref{D-pro} is proved.
\end{proof}
\begin{Remark}
{\rm (i)} As a direct application of Lemma \ref{lem-CD-pro}(i), the classical \emph{REC} is satisfied by $X$ with high probability if $\Sigma^{\frac12}$ satisfies the classical \emph{REC}(s,t,a) and \begin{equation*}
m>\frac{c_1\zeta(\Sigma)}{\phi_1^2(s,t,a,\Sigma^{\frac{1}{2}})}\left(\sqrt{s+t}+a\sqrt{s}\right)^2\log n,
\end{equation*}
which covers \cite[Corollary 1]{Raskutti2010Restricted} as a special case when $t=0$.

{\rm (ii)} Recall from Remark \ref{rmk-sqrtm} that $\phi_q(s,t,a,X)$ given by \eqref{eq-rec} usually scales as $\sqrt{m}$ independent of $s$ and $n$ for the Gaussian random design $X$. Then Lemma \ref{lem-CD-pro}(i) is applicable to indicating that $\phi_q(s,t,a,\Sigma^{\frac12})$ usually scales as a constant independent of $s$, $m$ and $n$.
\end{Remark}

Below, we consider the dominant property \eqref{dominant} in the situation when $X$ is a Gaussian random design. For the $\ell_q$ minimization problem $({\rm CP}_{q,\epsilon})$, Proposition \ref{prop-cone1} is still applicable for the case when $X$ is a Gaussian random design since it does not rely on the assumption of $X$, and thus, \eqref{I-16} holds with the same probability for the random design scenario; see Remark \ref{rem-domin}. In the following proposition, we show the dominant property \eqref{I-18} for the $\ell_q$ regularization problem $({\rm RP}_{q,\lambda})$ with a random design by virtue of Proposition \ref{prop-cone2}. Recall that $\epsilon$, $\lambda$, $\rho$ and the events $\mathscr{A}$ and $\mathscr{B}$ are given in the preceding section; see \eqref{rho}-\eqref{B-ev} for details.
\begin{Proposition}
Let $\hat{\beta}_{q,\lambda}$ be an optimal solution of $(\emph{RP}_{q,\lambda})$ with $\lambda$ given by \eqref{lambda-1}. Suppose that $\Sigma_{j,j}=1$ for all $j=1,\dots,n$. Then, there exist universal positive constants $(c_1,c_2)$ and $\tau\geq 1$ (independent of $m,n,q,a,\theta,b,\epsilon,r,\lambda,\Sigma$) such that, if
\begin{equation}\label{I-43}
m>\frac{c_1\tau^4}{\theta^2}\log n,
\end{equation}
then \eqref{I-18} holds with probability at least $(1-\left(n^b\sqrt{\pi\log n}\right)^{-1})(1-2\exp (-c_2\theta^2m/\tau^4))-\exp(-m)$.
\end{Proposition}
\begin{proof}
By \eqref{D-ev}, one sees by
Proposition \ref{prop-cone2}
that \eqref{I-18} holds under the event $\mathscr{A}\cap \mathscr{B}\cap \mathscr{D}$. Then it remains to estimate $\mathbb{P}(\mathscr{A}\cap \mathscr{B}\cap \mathscr{D})$.
By Lemma \ref{lem-CD-pro}(ii), there exist universal positive constants $(c_1,c_2)$ and $\tau\geq 1$ such that
\begin{equation*}
\mathbb{P}(\mathscr{D})\geq 1-2\exp (-c_2\theta^2m/\tau^4),
\end{equation*}
whenever $m$ satisfies \eqref{I-43}.
From Lemma \ref{lem-AB-pro} (cf. \eqref{B-pro}), we have also by \eqref{D-ev} that
\begin{equation*}
\mathbb{P}(\mathscr{B}|\mathscr{D})\geq 1-(n^b\sqrt{\pi\log n})^{-1}.
\end{equation*}
Then, it follows that
\begin{equation*}
\begin{aligned}
\mathbb{P}(\mathscr{B}\cap \mathscr{D})\
&= \mathbb{P}(\mathscr{B}|\mathscr{D})\mathbb{P}(\mathscr{D})\\
&\geq (1-(n^b\sqrt{\pi\log n})^{-1})(1-2\exp (-c_2\theta^2m/\tau^4)),
\end{aligned}
\end{equation*}
and then by the elementary probability theory and \eqref{A-pro} that,
\begin{equation*}
\begin{aligned}
\mathbb{P}(\mathscr{A}\cap \mathscr{B}\cap \mathscr{D})
&=
\mathbb{P}(\mathscr{B}\cap \mathscr{D})-\mathbb{P}(\mathscr{B}\cap \mathscr{D}\cap \mathscr{A}^c)\\
&\geq
\mathbb{P}(\mathscr{B}\cap \mathscr{D})+\mathbb{P}(\mathscr{A})-1\\
&\geq
\left(1-\left(n^b\sqrt{\pi\log n}\right)^{-1}\right)(1-2\exp (-c_2\theta^2m/\tau^4))-\exp(-m),
\end{aligned}
\end{equation*}
whenever $m$ satisfies \eqref{I-43}.
The proof is complete.
\end{proof}
\indent Now we are ready to present the main theorems of this section, in which we establish the $\ell_2$ recovery bounds for $(\text{CP}_{q,\epsilon})$ and $(\text{RP}_{q,\lambda})$ when $X$ is a Gaussian random design.
The first theorem illustrates the stable recovery capability of the $\ell_q$ minimization problem $(\text{CP}_{q,\epsilon})$ (within a tolerance proportional to the noise) with high probability when the design matrix is random as long as the vector $\beta^*$ is sufficiently sparse and the sample size $m$ is sufficiently large.
\begin{Theorem}\label{thm-CP-randomX}
Let $\bar{\beta}_{q,\epsilon}$ be an optimal solution of $(\emph{CP}_{q,\epsilon})$ with $\epsilon$ given by \eqref{rho}. Suppose that $\Sigma^{\frac{1}{2}}$ satisfies the $q$-\emph{REC}$(s,t,1)$. Then, there exist universal positive constants $(c_1,c_2)$ (independent of $m,n,q,s,t,\epsilon,\Sigma$) such that, if \eqref{REC-sample} is satisfied,
then it holds with probability at least $(1-\exp(-m))(1-\exp(-c_2m))$ that
\begin{equation}\label{I-48}
\|\bar{\beta}_{q,\epsilon}-\beta^*\|_2^2\leq
\frac{16(1+\left(\frac{s}t\right)^{\frac2q-1})}{m\phi_q^2(s,t,1,\Sigma^{\frac{1}{2}})}\epsilon^2.
\end{equation}
\end{Theorem}
\begin{proof}
To simplify the proof, corresponding to inequalities \eqref{I-27} and \eqref{I-48}, we define the following two events
\begin{equation*}
\begin{aligned}
\mathscr{E}_1&:=\left\{\|\bar{\beta}_{q,\epsilon}-\beta^*\|_2^2\leq 
\frac{4(1+\left(\frac{s}t\right)^{\frac2q-1})}{\phi_q^2(s,t,1,X)}\epsilon^2\right\},\\
\mathscr{E}_2&:=\left\{\|\bar{\beta}_{q,\epsilon}-\beta^*\|_2^2\leq
\frac{16(1+\left(\frac{s}t\right)^{\frac2q-1})}{m\phi_q^2(s,t,1,\Sigma^{\frac{1}{2}})}\epsilon^2\right\}.
\end{aligned}
\end{equation*}
Then, by the definition of $\mathscr{C}_1$ \eqref{C-ev}, we have that $\mathscr{C}_1\cap \mathscr{E}_1\subseteq \mathscr{E}_2$ and thus
\begin{equation}\label{neq1-thm3}
\mathbb{P}(\mathscr{E}_2)\geq \mathbb{P}(\mathscr{E}_1\cap \mathscr{C}_1)=\mathbb{P}(\mathscr{E}_1|\mathscr{C}_1)\mathbb{P}(\mathscr{C}_1).
\end{equation}
Note by Theorem \ref{thm-CP} that
\begin{equation}\label{neq2-thm3}
\mathbb{P}(\mathscr{E}_1|\mathscr{C}_1)\geq 1-\exp(-m).
\end{equation}
By Lemma \ref{lem-CD-pro}(i) (with $a=1$), there exist universal positive constants $(c_1,c_2)$ such that \eqref{REC-sample} ensures
\eqref{C-pro}.
Then we obtain by \eqref{neq1-thm3} and \eqref{neq2-thm3} that
\begin{equation*}
\mathbb{P}(\mathscr{E}_2)\geq (1-\exp(-m))(1-\exp(-c_2m)),
\end{equation*}
whenever $m$ satisfies \eqref{REC-sample}.
The proof is complete.
\end{proof}
\indent As a direct application of Theorem \ref{thm-CP-randomX} to the special case when $q=1$, the following corollary presents the $\ell_2$ recovery bound of the $\ell_1$ minimization problem $(\text{CP}_{1,\epsilon})$ with a Gaussian random design as
\begin{equation*}
\|\bar{\beta}_{1,\epsilon}-\beta^*\|_2=O(\epsilon)
\end{equation*}
under the classical REC.
\begin{Corollary}\label{corol-CP-randomX}
Let $\bar{\beta}_{1,\epsilon}$ be an optimal solution of $(\emph{CP}_{1,\epsilon})$ with $\epsilon$ given by \eqref{rho}. Suppose that $\Sigma^{\frac{1}{2}}$ satisfies the $1$-\emph{REC}$(s,t,1)$.
Then, there exist universal positive constants $(c_1,c_2)$ (independent of $m,n,q,s,t,\epsilon,\Sigma$) such that, if
\begin{equation*}
m>\frac{c_1\zeta(\Sigma)}{\phi_1^2(s,t,1,\Sigma^{\frac{1}{2}})}(\sqrt{s+t}+\sqrt{s})^2\log n,
\end{equation*}
then it holds with probability at least $(1-\exp(-m))(1-\exp(-c_2m))$ that
\begin{equation*}
\|\bar{\beta}_{1,\epsilon}-\beta^*\|_2^2\leq
\frac{16(1+\frac{s}t)}{m\phi_1^2(s,t,1,\Sigma^{\frac{1}{2}})}\epsilon^2.
\end{equation*}
\end{Corollary}

The other main theorem of this section is as follows, in which we exploit the estimation of prediction loss, the oracle property and the $\ell_2$ recovery bound of parameter approximation of the $\ell_q$ regularization problem $(\text{RP}_{q,\lambda})$ with a Gaussian random design under the $q$-REC of the square root of its population covariance matrix.
\begin{Theorem}\label{thm-RP-randomX}
Let $\hat{\beta}_{q,\lambda}$ be an optimal solution of $(\emph{RP}_{q,\lambda})$ with $\lambda$ given by \eqref{lambda-1}. Suppose that $\Sigma_{j,j}=1$ for all $j=1,\dots,n$ and $\Sigma^{\frac{1}{2}}$ satisfies the $q$-\emph{REC}$(s,t,a)$.  Then, there exist universal positive constants $(c_1,c_2,c_3,c_4)$ and $\tau\geq 1$ (independent of $m,n,q,s,t,a,\theta,b,\epsilon,r,\lambda,\Sigma$) such that, if
\begin{equation}\label{I-51}
m>\max\left\{ \frac{c_1(\sqrt{s+t}+a^{\frac1q}\sqrt{s}\left(\frac{s}t\right)^{\frac1q-1})^2}{\phi_q^2(s,t,a,\Sigma^{\frac{1}{2}})}\log n,\ \frac{c_3\tau^4}{\theta^2}\log n\right\},
\end{equation}
then it holds with probability at least $$\left(1-\exp(-m)-\left(n^b\sqrt{\pi\log n}\right)^{-1}\right)(1-\exp(-c_2m)-2\exp (-c_4\theta^2m/\tau^4))$$ that
\begin{equation}\label{I-52}
\frac{1}{m}\|X\hat{\beta}_{q,\lambda}-X\beta^*\|_2^2\leq \left(\frac{2^{q+1}a\lambda}{\phi_q^q(s,t,a,\Sigma^{\frac12})}\right)^{\frac{2}{2-q}}s,
\end{equation}
\begin{equation}\label{I-53}
\frac{1}{2m}\|X\hat{\beta}_{q,\lambda}-X\beta^*\|_2^2+\lambda\|(\hat{\beta}_{q,\lambda})_{ J^c}\|_q^q\leq \left(\frac{8^{\frac{q}2}a\lambda}{\phi_q^q(s,t,a,\Sigma^{\frac{1}{2}})}\right)^{\frac{2}{2-q}}s,
\end{equation}
\begin{equation}\label{I-54}
\|\hat{\beta}_{q,\lambda}-\beta^*\|_2^2\leq \left(1+a^{\frac2q}\left(\frac{s}{t}\right)^{\frac2q-1}\right)\left(\frac{8a\lambda}{\phi_q^2(s,t,a,\Sigma^{\frac{1}{2}})}\right)^{\frac{2}{2-q}}s.
\end{equation}
\end{Theorem}
\begin{proof}
To simplify the proof,
we define the following six events
\begin{equation*}
\begin{aligned}
\mathscr{F}_1&=\left\{\eqref{I-31}\ \text{happens}\right\},\quad
\mathscr{F}_2=\left\{\eqref{I-32}\ \text{happens}\right\}, \quad
\mathscr{F}_3=\left\{\eqref{I-33}\ \text{happens}\right\},\\
\mathscr{G}_1&=\left\{\eqref{I-52}\ \text{happens}\right\},\quad
\mathscr{G}_2=\left\{\eqref{I-53}\ \text{happens}\right\},\quad
\mathscr{G}_3=\left\{\eqref{I-54}\ \text{happens}\right\}.
\end{aligned}
\end{equation*}
Fix $i\in \{1,2,3\}$. Then, we have by \eqref{C-ev} that $\mathscr{C}_{a}\cap \mathscr{F}_i\subseteq \mathscr{G}_{i}$
and thus
\begin{equation}\label{neq1-thm4}
\mathbb{P}(\mathscr{G}_{i})\geq \mathbb{P}(\mathscr{C}_{a}\cap \mathscr{F}_i).
\end{equation}
By Lemma \ref{lem-CD-pro}, there exist universal positive constants $(c_1,c_2,c_3,c_4)$ and $\tau\geq 1$ such that, \eqref{I-51} ensures \eqref{C-pro} and \eqref{D-pro}.
Then it follows from \eqref{C-pro} and \eqref{D-pro} that
\begin{equation}\label{Th2CD-pro}
\mathbb{P}(\mathscr{C}_{a}\cap \mathscr{D})\geq 
\mathbb{P}(\mathscr{C}_{a})+P(\mathscr{D})-1\geq 1-\exp(-c_2m)-2\exp(-c_4\theta^2m/\tau^4),
\end{equation}
whenever $m$ satisfies \eqref{I-51}.
Recall from Theorem \ref{thm-RP} that
\begin{equation*}
\mathbb{P}(\mathscr{F}_i|\mathscr{C}_{a}\cap \mathscr{D})\geq 1-\exp(-m)-\left(n^b\sqrt{\pi\log n}\right)^{-1}.
\end{equation*}
This, together with \eqref{Th2CD-pro}, implies that
\begin{equation*}
\begin{aligned}
\mathbb{P}(\mathscr{C}_{a}\cap \mathscr{F}_i)
&\geq \mathbb{P}(\mathscr{F}_i|\mathscr{C}_{a}\cap \mathscr{D})\,\mathbb{P}(\mathscr{C}_{a}\cap \mathscr{D})\\
&\geq \left(1-\exp(-m)-\left(n^b\sqrt{\pi\log n}\right)^{-1}\right)(1-\exp(-c_2m)-2\exp (-c_4\theta^2m/\tau^4)).
\end{aligned}
\end{equation*}
Then, one has by \eqref{neq1-thm4} that
\begin{equation*}
\mathbb{P}(\mathscr{G}_i)\geq
\left(1-\exp(-m)-\left(n^b\sqrt{\pi\log n}\right)^{-1}\right)(1-\exp(-c_2m)-2\exp(-c_4\theta^2m/\tau^4)),
\end{equation*}
whenever $m$ satisfies \eqref{I-51}.
The proof is complete.
\end{proof}
\indent As an application of Theorem \ref{thm-RP-randomX} to the special case when $q=1$ and $a=3$, the following corollary presents the statistical properties of the $\ell_1$ regularization problem with a Gaussian random design under the classical REC. A similar $\ell_2$ recovery bound was shown in \cite[Theorem 3.1]{Zhou2009Restricted} by using a different analytic technique.
\begin{Corollary}\label{corol-RP-randomX}
Let $\hat{\beta}_{1,\lambda}$ be an optimal solution of $(\emph{RP}_{1,\lambda})$ with
\[
\lambda=2\sigma(1+\theta) \sqrt{\frac{2(1+b)\log n}{m}}.
\]
Suppose that $\Sigma_{j,j}=1$ for all $j=1,\dots,n$ and $\Sigma^{\frac{1}{2}}$ satisfies the $1$-\emph{REC}$(s,t,3)$.
Then, there exist universal positive constants $(c_1,c_2,c_3,c_4)$ and $\tau\geq 1$ (independent of $m,n,s,t,\theta,b,\Sigma$) such that, if 
\begin{equation*}
m>\max\left\{ \frac{c_1(\sqrt{s+t}+3\sqrt{s})^2}{\phi_1^2(s,t,3,\Sigma^{\frac{1}{2}})}\log n,\ \frac{c_3\tau^4}{\theta^2}\log n\right\},
\end{equation*}
then it holds with probability at least $$(1-\exp(-m)-\left(n^b\sqrt{\pi\log n}\right)^{-1})(1-\exp(-c_2m)-2\exp (-c_4\theta^2m/\tau^4))$$ that
\begin{equation*}
\frac{1}{m}\|X\hat{\beta}_{1,\lambda}-X\beta^*\|_2^2\leq
\frac{1152(1+b)(1+\theta)^2}{\phi_1^2(s,t,3,\Sigma^{\frac{1}{2}})}\frac{s\log n}{m}\sigma^2,
\end{equation*}
\begin{equation*}
\frac{1}{2m}\|X\hat{\beta}_{1,\lambda}-X\beta^*\|_2^2+\lambda\|(\hat{\beta}_{1,\lambda})_{ J^c}\|_1\leq \frac{576(1+b)(1+\theta)^2}{\phi_1^2(s,t,3,\Sigma^{\frac{1}{2}})}\frac{s\log n}{m}\sigma^2,
\end{equation*}
\begin{equation*}
\|\hat{\beta}_{1,\lambda}-\beta^*\|_2^2\leq 
\frac{4608(1+b)(1+\theta)^2(1+9\frac{s}t)}{\phi_1^4(s,t,3,\Sigma^{\frac{1}{2}})}\frac{s\log n}{m}\sigma^2.
\end{equation*}
\end{Corollary}

\section{Numerical Experiments}
\noindent The purpose of this section is to carry out the numerical experiments to illustrate the stability of the $\ell_q$ optimization methods, verify the established theory of the $\ell_2$ recovery bounds in the preceding sections and compare the numerical performance of the $\ell_q$ regularization methods with another two widely used nonconvex regularization methods, namely the SCAD and MCP. In particular, we are concerned with the cases when $q=0$, $1/2$, $2/3$ and $1$. To solve the $\ell_q$ minimization problems, we will apply the iterative reweighted algorithm \cite{Candes2008Enhancing, Chartrand2008Iteratively}. To solve the $\ell_q$ regularization problems, we will apply the iterative hard thresholding algorithm \cite{Blumensath2008Iterative} for $q=0$, the proximal gradient algorithm \cite{Hu2017Group} for $q=1/2$ and $2/3$, and FISTA \cite{Beck2009Fast} for $q=1$, respectively. The proximal gradient algorithm proposed in \cite{Loh2015Regularized} will be used to solve the SCAD and MCP. All numerical experiments are performed in MATLAB R2014b and executed on a personal desktop (Intel Core i7-4790, 3.60 GHz, 8.00 GB of RAM).\\
\indent The simulated data are generated via a standard process; see, e.g., \cite{Agarwal2012Fast, Hu2017Group}. Specifically, we randomly generate an i.i.d. Gaussian ensemble $X\in \R^{m\times n}$ and a sparse vector $\beta^*\in \R^n$ with the sparsity being equal to $s$.
The observation $y$ is then generated by the MATLAB script
\begin{equation*}
y=X*\beta^*+sigma*randn(m,1),
\end{equation*}
where $sigma$ is the noise level, that is the standard deviation of Gaussian noise. In the numerical experiments, the dimension of variables and the noise level are set as $n=1024$ and $sigma=0.01$, respectively.\\
\indent For each sparsity level, which is $s/n$, we randomly generate the data $X$, $\beta^*$, $y$ for 100 times and run the algorithms mentioned above to solve the $\ell_q$ optimization problems for $q=0$, $1/2$, $2/3$ and 1 as well as the SCAD and MCP. The parameter $\epsilon$ in the $\ell_q$ minimization methods $\text{(CP)}_{q,\epsilon}$ is set as $\epsilon=sigma*\sqrt{m+2\sqrt{2m}}$ in order to guarantee that $\|e\|_2^2$ is no more than $\epsilon^2$ with overwhelming probability \cite{Candes2006Stable, Candes2008Enhancing}.  The parameter $\lambda$ in the $\ell_q$ regularization methods $(\text{RP}_{q,\lambda})$ is chosen by 10-fold cross validation. To simplify the notations, the solution of different problems will all be denoted as $\hat{\beta}$. In order to reveal the dependence of $\ell_2$ recovery bounds on sample size and inspired by the established theorems in the preceding sections (e.g., \eqref{I-51}), we report the numerical results for a range of sample sizes of the form $m=\Omega(s\log n)$.

The experiment is implemented to study the performance on variable selection of the $\ell_q$ regularization methods as well as the SCAD and MCP. We use following two criteria to characterize the capability of variable selection:
\[
\text{sensitivity}=\frac{\text{true positive}}{\text{true positive+false negative}}
\quad \mbox{and}\quad
\text{specificity}=\frac{\text{true negative}}{\text{true negative+false positive}},
\]
which respectively measures the proportion of positives and negatives that are correctly identified. The larger values of both sensitivity and specificity mean the higher capability of variable selection. The results are illustrated by averaging over the 100 random trials. Tables \ref{t-sen} and \ref{t-spe} respectively chart the sensitivity and specificity of these methods at a sparsity level 10\%. It is illustrated that the sensitivity and specificity of all these methods increase as the sample size grows, except for the specificity of Lasso, which is resulted from the fact that there are many small nonzero coefficients estimated by Lasso. We also note that the lower-order regularization method (e.g., when $q=1/2,2/3$) outperforms the other regularization methods in the sense that it can almost completely select the true model when the size of samples is getting large.
\begin{table}[htbp]
\centering
\caption{Sensitivity of different regularization methods.}
\begin{tabular}{|c|c|c|c|c|c|c|}
\hline
\multirow{2}{*}{Method} & \multicolumn{6}{c|}{Sample size}                    \\ \cline{2-7}
                        & 177    & 355    & 532    & 710    & 887    & 976    \\ \hline
q=0                     & 0.3029 & 0.8931 & 0.9824 & 0.9873 & 0.9902 & 0.9912 \\ \hline
q=1/2                   & 0.2902 & 0.5108 & 0.9412 & 0.9873 & 0.9892 & 0.9941 \\ \hline
q=2/3                   & 0.3108 & 0.9333 & 0.9922 & 0.9931 & 0.9941 & 0.9971 \\ \hline
q=1                     & 0.5088 & 0.9980 & 1.0000 & 1.0000 & 1.0000 & 1.0000 \\ \hline
SCAD                    & 0.2882 & 0.8471 & 0.9157 & 0.9363 & 0.9324 & 0.9422 \\ \hline
MCP                     & 0.1373 & 0.4539 & 0.8461 & 0.9088 & 0.9353 & 0.9382 \\ \hline
\end{tabular}\label{t-sen}
\end{table}

\begin{table}[htbp]
\centering
\caption{Specificity of different regularization methods.}
\begin{tabular}{|c|c|c|c|c|c|c|}
\hline
\multirow{2}{*}{Method} & \multicolumn{6}{c|}{Sample size}                    \\ \cline{2-7}
                        & 177    & 355    & 532    & 710    & 887    & 976    \\ \hline
q=0                     & 0.9229 & 0.9882 & 0.9980 & 0.9986 & 0.9989 & 0.9990 \\ \hline
q=1/2                   & 0.8810 & 0.8119 & 1.0000 & 1.0000 & 1.0000 & 1.0000 \\ \hline
q=2/3                   & 0.8782 & 0.9999 & 1.0000 & 1.0000 & 1.0000 & 1.0000 \\ \hline
q=1                     & 0.8088 & 0.7454 & 0.7680 & 0.7473 & 0.7357 & 0.6120 \\ \hline
SCAD                    & 0.9653 & 0.9900 & 0.9906 & 0.9908 & 0.9919 & 0.9925 \\ \hline
MCP                     & 0.9466 & 0.9757 & 0.9909 & 0.9919 & 0.9925 & 0.9925 \\ \hline
\end{tabular}\label{t-spe}
\end{table}

Finally, it is worth mentioning that the existing $\ell_q$ optimization algorithms (see, e.g., \cite{Candes2008Enhancing, Chartrand2008Iteratively, Hu2017Group, Xu2012Regu}) are only proved to converge to a critical point, while their convergence to a global optimum is still an open question. Nevertheless, it is demonstrated by the numerical results above, as well as the ones in the literature, that the limiting point of these algorithms performs well in estimating the underlying true parameter.

\end{document}